%% file: main.tex
\documentclass[a4paper,11pt]{article}
\include{preamble}
\title{On the Limitations of the 
Univariate Marginal Distribution Algorithm to Deception
and Where Bivariate EDAs might help\footnote{Preliminary 
version of this work will appear
in the Proceedings of 15th ACM/SIGEVO Workshop on Foundations 
of Genetic Algorithms (FOGA XV),
Potsdam, Germany}
}
\author{Per Kristian Lehre \& Phan Trung Hai Nguyen\\
School of Computer Science\\
University of Birmingham\\
Birmingham B15 2TT, United Kingdom
}

\begin{document}
\maketitle

\begin{abstract}
We introduce a new benchmark
problem called Deceptive Leading Blocks (\dlblock)
to rigorously study the 
runtime of the Univariate Marginal Distribution
Algorithm (\umda) in the presence of epistasis and deception.
We show that simple Evolutionary Algorithms (\eas) outperform the \umda 
unless the 
selective pressure $\mu/\lambda$ is extremely high, where $\mu$ and 
$\lambda$ are the parent and offspring population sizes, respectively.
More precisely, we show that the \umda with a parent population size of 
$\mu=\Omega(\log n)$ has an 
expected runtime of $e^{\Omega(\mu)}$ 
on the \dlblock problem assuming any selective pressure
$\frac{\mu}{\lambda}\geq \frac{14}{1000}$,
as opposed to the expected runtime of
$\bigO{n\lambda\log \lambda+n^3}$ for the non-elitist
$(\mu,\lambda)~\ea$ with $\mu/\lambda\leq 1/e$.
These results illustrate inherent limitations of univariate 
\edas against deception and epistasis, which are common characteristics
of real-world problems. In contrast, empirical evidence reveals the efficiency of the 
bi-variate \mimic algorithm on the \dlblock  problem.
Our results suggest that one should consider \edas with
more complex probabilistic models when optimising problems with some 
degree of epistasis and deception.
\end{abstract}

\section{Introduction}

Estimation of distribution algorithms (\edas)
\cite{Muhlenbein1996,Pelikan2002,larranaga2001estimation}
are a class of randomised search heuristics
with many real-world applications (see \cite{bib:Hauschild2011}
and references therein). Unlike traditional \eas, which 
define \textit{implicit} models of promising solutions 
via genetic operations such as 
crossover and mutation, \edas optimise objective functions 
by constructing and sampling \textit{explicit} probabilistic models 
to generate \textit{offspring} for the next iteration. 
The workflow of \edas is an iterative process, where the initial model is a
uniform distribution over the search space. The starting population consists of 
$\lambda$ \textit{individuals} sampled from the uniform distribution. 
A fitness function then scores each individual, 
and the algorithm selects the $\mu$ fittest individuals 
to update the model (where $\mu<\lambda$). The procedure is 
repeated until some termination condition
is fulfilled, which is usually a threshold on the number of iterations or
on the quality of the fittest offspring 
\cite{bib:Hauschild2011, Eiben:2003:IEC:954563}.

Many variants of 
\edas have been proposed over the last decades.
They differ in the way their models are represented, 
updated as well as sampled over iterations. 
In general, \edas are  categorised
into two main classes: \textit{univariate} 
and \textit{multivariate}. Univariate \edas 
take advantage of first-order statistics (i.e. the mean) to 
build a probability vector-based
model and assume independence between
decision variables. The probabilistic model is represented as 
an $n$-vector, where each component is called a \textit{marginal} (also 
\textit{frequency}) and $n$ is the problem instance size. Typical univariate
\edas are compact Genetic 
Algorithm (\cga \cite{bib:Harik1999}), Univariate 
Marginal Distribution Algorithm (\umda \cite{Muhlenbein1996}) and 
Population-Based Incremental Learning 
(\pbil \cite{bib:Baluja1994}). 
In contrast, multivariate \edas apply higher-order 
statistics to model the correlations
between decision variables of the addressed problems. 

The \cga is the 
simplest \eda, which operates on a population of two 
individuals and updates the probabilistic model
additively via a parameter $K$, which is 
often referred to as the \textit{hypothetical population size} of a 
genetic algorithm that the \cga is supposed to model. 
The two individuals are compared in terms of fitness to find the winner, and an 
increase of $\pm 1/K$ takes place at bit positions where individuals have different values. 
The algorithm also restricts the 
marginals to be within an interval 
$[1/n,1-1/n]$, where the values $1/n$ and $1-1/n$ are called
the \textit{lower} and \textit{upper borders} (or margins), 
respectively, in order to 
prevent the marginals from fixing at trivial 
borders, which may cause the algorithm to converge prematurely.
Such an algorithm is referred to as a \cga with margins.
In contrast, the \umda
is another univariate \eda that
has a larger population of $\lambda$ individuals. 
In each so-called iteration, the marginal is renewed/updated to the 
frequency of 1-bit among the $\mu$ fittest individuals at each bit position (also called \textit{empirical frequency}). 
Unlike the \cga, whose marginals can only get increased by an amount of 
$1/K$, those of the \umda might jump between the upper and lower borders.
A generalisation of the \umda is the \pbil, where each
marginal is updated following a convex combination of the 
current marginal and the empirical frequency via
a so-called \textit{smoothing parameter}. 
The marginals cannot change
by a large amount, and it is then less likely 
that \textit{genetic drift} \cite{bib:Sudholt2016} happens soon
after the algorithm starts. 

The theory of evolutionary computation literature provides
rigorous analyses giving
insights into  the  \textit{runtime} 
(synonymously, \textit{optimisation time}), that is 
the number of function evaluations of the
studied algorithm until an optimal solution is sampled for the 
first time. 
In other words, theoretical work 
usually addresses the unlimited case
when we consider the run of the algorithm as an infinite process.
These analyses provide \textit{performance guarantees}
of the algorithm for a wide range of problem instance size. 

Although \edas were introduced several
decades ago and have since shown strong potential
as a global optimiser via
many practical applications \cite{bib:Hauschild2011}, 
the theoretical understanding of \edas is very limited.
There had been only a handful of  
runtime analyses of \edas by 2015. Recently, 
 they have drawn more attention from the community \cite{Doerr2019-cGA-JUMP,bib:Friedrich2016,bib:Krejca,Doerr2019-cGA-exponential, bib:Sudholt2016,bib:Wu2017,bib:Witt2017,CGA-GAUSSIAN-NOISE,bib:Lehre2017,bib:lehre2018a,
Lehre2018PBIL,Witt2018Domino,
bib:Lengler2018,DoerrSigEDA2018,Hasenohrl:2018}. 
While rigorous runtime analyses provide deep insights into
the performance of randomised search heuristics, it is highly
challenging even for simple algorithms on toy functions. 
Most current runtime results merely concern 
 univariate \edas on functions like 
\om \cite{bib:Krejca,bib:Witt2017,bib:Lehre2017,bib:Wu2017,bib:Lengler2018}, 
\los \cite{DoerrSigEDA2018,bib:Friedrich2016,Lehre2018PBIL, bib:Wu2017,bib:Lehre2019}, 
\bval \cite{Witt2018Domino,Lehre2018PBIL} 
and \textsc{Jump} \cite{Hasenohrl:2018,Doerr2019-cGA-JUMP, Doerr2019-cGA-exponential}, 
hoping that this provides valuable insights into the
development of new techniques for analysing multivariate variants of \edas and
the behaviour of such algorithms on easy parts of more complex
problem spaces \cite{Doerr2016}.
There are two main reasons accounted for this.
Firstly, the working principle of these algorithms, which 
are not originally designed to support the theoretical analyses,
is complicated, involving lots of randomnesses, and the interplay
of decision variables usually has a huge impact on 
the overall runtime \cite{Doerr2016,Sudholt:2017:CSU}. 
Secondly, we are lacking in the state-of-the-art 
tools in algorithmics \cite{Doerr2016}. There are seemingly two
currently popular techniques used to analyse \edas. The first tool is 
\textit{drift theorems} 
\cite{HE200359,Doerr:2010:MDA,JohanmsemPhDThesis}, while
the other tool is the \textit{level-based theorem}, 
first proposed in \cite{Lehre:2011:FNP} and constantly 
improved upon \cite{bib:Corus2016, Doerr2019-MulUpTheorem}, in 
the context of non-elitist population-based 
\eas. 

The fact that the \umda never attempts to 
learn variable interactions leads some to conjecture that 
the algorithm will not perform well in environments with
\textit{epistasis} and \textit{deception}.
More specifically, epistasis corresponds to the 
maximum number of other variables each of the $n$ decision
variables depends on \cite{DAVIDOR-epistasis, ROBIN-CEC-2008}. 
We can take as an example 
the \los function, which has a maximum epistasis level of $n$
and is often used 
to study the ability of \eas to cope 
with variable dependency.
Previous studies \cite{Dang:2015,Lehre2018PBIL,bib:Lehre2019} show  an
$\bigO{n^2}$ expected
 runtime  for the \umda and the \pbil 
 on this function, which seems to contradict the above-mentioned claim. However, 
 \citet{bib:Lehre2019} recently showed that 
univariate \edas based on probability vectors 
have certain limitations to epistasis: 
if the \textit{selective pressure} is larger 
than roughly $1/e$ 
(as previously required in \cite{Dang:2015,Lehre2018PBIL}) and the parent population is sufficiently large, 
the \umda fails to optimise the \los function
in polynomial expected runtime. 

Regarding deception, an example 
was already mentioned in \cite{bib:Hauschild2011}, where the \umda
gets stuck in the concatenated trap of order 5 
(called \textsc{Trap-5} \cite{ACKLEY-1987},  
 where the original trap
function is applied to each of $\floor{n/5}$ 
blocks of five consecutive bits) 
since marginals are deceived to
hit the lower border. 
This excellent example demonstrates the limitations
of the univariate model as statistics of low order
is misleading \cite{bib:Hauschild2011}. However,  
this function might be as difficult for the \umda as it is 
for the $\eas$ \cite{ROBIN-CEC-2008}. We believe not only 
will the \umda fail on this function, but it will also fail on some
other problem with a milder degree of deception. To this end,
a function where the \umda takes an 
exponential expected runtime,
while simple \eas has a
polynomial expected runtime 
is still missing\footnote{We define the terms ``polynomial'' and 
``exponential'' as $n^{\bigO{1}}$ and 
$2^{n^{\Omega(1)}}$, respectively.}.
On the other hand, a 
function where the \umda outperforms 
simple \eas does exist.  
\citet{bib:Chen2009a} proposed the   
\textsc{Substring} function to point out
the advantage of the probability vector-based model. More specifically, 
the $(1+1)~\ea$ needs a runtime of
$2^{\Omega(n)}$ with probability $1-2^{-\Omega(n)}$, 
whereas the \umda with $\lambda=\Omega(n^{2+\varepsilon})$, for any
constant $\varepsilon>0$, and $\mu=\lambda/2$ optimises the function
in polynomial runtime with probability $1-2^{-\Omega(n)}$.
We note that this result is very limited in many senses that
the population size is large, 
the selective pressure is fixed to $1/2$ and the 
considered \umda does not have borders.

Motivated by this, we introduce a new benchmark problem 
which  has a maximum epistasis 
level of $n$ and is mildly deceptive.
The fitness depends on the number of leading 11s, 
and reaching a unique global optimum requires overcoming many mild traps. 
Generally speaking, this function is harder than the \los function, but
still much easier compared to the \textsc{Trap-5} function.
The problem, which we call \emph{Deceptive Leading 
Blocks} (\dlblock), can be formally 
defined over a finite binary 
search space $\mathcal{X}\coloneqq  \{0,1\}^n$ as follows.
\begin{displaymath}
\dlblock(x) \coloneqq   \begin{cases}
n &\text{if } \phi(x) = n/2,\\
2\cdot \phi(x) + 1 &\text{if } x_{2\phi(x)+1} + x_{2\phi(x)+2} = 0,\\
2\cdot \phi(x) &\text{if } x_{2\phi(x)+1} + x_{2\phi(x)+2}=1,\\
\end{cases} 
\end{displaymath}
where 
\begin{equation}\label{eq:leading-blocks}
\phi(x) \coloneqq   \sum_{i=1}^{n/2} \prod_{j=1}^{2i} x_j
\end{equation}
denotes the number of leading 11s, which 
is identical to $\textsc{LOB}_2(x)$ 
in \cite[Definition 13]{Jansen:2004}. 
The global optimum is the all-ones bitstring.
The bitstring is partitioned into independent blocks of 
$w\ge 2$ consecutive bits.  
The difficulty of the problem is determined by
the width $w$ of each block, that can be altered to increase the level of deception. 
Here, we consider the smallest width
$w=2$, as this suffices to prove an exponential gap between the
runtimes of the \umda
and the ($\mu$,$\lambda$) EA. \dlblock is similar to the \los function 
except it attempts to deceive the algorithm by 
assigning tricky weights to different settings of 
the leftmost non-11 block, which we  call the \emph{active block}. 
Each leading 11
contributes a value of two, while 
a value of one is awarded if the active block is a 
00, and no reward is given for any other blocks.
The all-ones bitstring has a fitness of $(n/2)\cdot 2=n$
since there are $n/2$ blocks, assuming that $n$ 
is a multiple of $w=2$. 

By studying this problem, we first show that
simple \eas, including elitist and non-elitist variants
like \onelambdaEA, \muoneEA and $(\mu,\lambda)~\ea$, and some Genetic 
Algorithms,  can optimise the \dlblock
function within an $\bigO{n^3}$ expected runtime. 
We then show that the \umda fails
to optimise this function in polynomial 
expected runtime, assuming that the selective pressure is $\mu/\lambda\geq\frac{14}{1000}$,
and the parent population size $\mu$ is sufficiently large.
More specifically, the expected runtime is 
$n^{\Omega(1)}$ 
 when $\mu=\Theta(\log n)$, while a lower bound of 
$2^{\Omega(n^\varepsilon)}$ is achieved for 
$\mu\ge cn^\varepsilon$ for some constants $c, \varepsilon>0$.
In the latter case, we say the \umda is 
fooled  by deceptive fitness. On the other hand, if the 
 selective pressure is in the order of 
 $1/\mu$ (i.e., extremely high), we obtain an upper bound of 
 $\bigO{n^3+n\lambda\log \lambda}$ 
 on the expected runtime of the \umda on the \dlblock function. 
Intuitively speaking, under this extreme selective pressure the \umda selects few fittest
individuals to update the probabilistic model, 
and we believe (with some empirical evidence in Section~\ref{sec:experiments}) 
that the \umda degenerates into 
the $(1,\lambda)~\ea$, where only the fittest offspring 
is selected to the next generation. Table~\ref{tab:freq} summarises 
the main results in this paper.

 \begin{table}
  \caption{Expected runtime of some simple \eas and the \umda
  (with borders) on the
  \dlblock function.}
  \label{tab:freq}
  \centering
  \begin{tabular}{@{}llll@{}}
    \toprule
    Algorithm & Sel. Pressure & Pop. Size & Expected Runtime\\
    \midrule
    $(1+\lambda)~\ea$ &--&--& $\mathcal{O}(n\lambda +n^3)$\\
    $(\mu+1)~\ea$ &--&--&$\mathcal{O}(\mu n\log n + n^3)$\\
    $(\mu,\lambda)~\ea$ &$\mu/\lambda=\mathcal{O}(1)$&$\lambda=\Omega(\log n)$&$\mathcal{O}(n\lambda\log\lambda+n^3)$\\
    \umda & $\mu/\lambda=\mathcal{O}(1/\mu)$& $\mu=\Omega(\log n)$& $\mathcal{O}(n\lambda\log\lambda+n^3)$ \\
    & $\mu/\lambda>\frac{14}{1000}$& $\mu= \Omega(\log n)$& $e^{\Omega(\mu)}$\\
  \bottomrule
\end{tabular}
\end{table}

Last but not least, many algorithms
similar to the \umda with 
fitness proportional selection have 
a wide range of applications in
\textit{bioinformatics} \cite{Armaaanzas2008}.
The algorithms relate to the notion of \textit{linkage
equilibrium} \cite{bib:Slatkin2008}
-- a popular model assumption in 
population genetics. Studying the \umda might
solidify our
understanding of population dynamics.
Based on results from
this paper, we believe that the \umda and 
other univariate model-based algorithms 
must be tuned carefully when optimising 
objective functions with epistasis and 
might not perform well in deceptive environments.

The paper is structured as follows. 
Section~\ref{sec:preliminaries} introduces the algorithms, 
including the \umda, $\eas$ and 
 the \mimic algorithm.  
Section~\ref{sec:eas-on-dlb} provides analyses of the expected runtime of simple 
\eas and GAs on the \dlblock function, followed by a 
detailed runtime analysis for the \umda on 
the \dlblock function.
Next, we illustrate the efficiency of the \mimic algorithm on the 
\dlblock function in
Section~\ref{sec:experiments} via a small empirical study.
Finally, Section~\ref{sec:conclusion} gives some concluding remarks
and suggests future work.

\section{Preliminaries}
\label{sec:preliminaries}

This section briefly describes the algorithms 
studied in this paper. Recall that
$\mathcal{X}=\{0,1\}^n$. 
Each \textit{individual} (or bitstring) is represented as 
$x=(x_1,x_2,\ldots,x_n) \in \mathcal{X}$. 
The \textit{population} of $\lambda$ individuals 
in an iteration $t\in \mathbb{N}$ is denoted as 
$P_t\coloneqq  (x_t^{(1)},\ldots,x_t^{(\lambda)})$. 
We consider in this paper the maximisation of
an \textit{objective function} 
$f:\mathcal{X} \rightarrow \mathbb{R}$. Denote 
$[n]\coloneqq \mathbb{N}\cap [1,n]$.

\subsection{Evolutionary algorithms}
The $(\mu+\lambda)~\ea$ is a mutation-only \ea, which 
operates on a population of $\mu$ individuals. 
In each iteration, the 
algorithm generates $\lambda$ new offspring 
by flipping bits in each of $\lambda$
individuals, chosen
uniformly at random from the population, independently with mutation rate 
$1/n$. Afterwards, the $\mu$ fittest individuals out of a pool of $\lambda+\mu$ 
individuals are selected to form the new population. 
The algorithm is \textit{elitist} since 
the best fitness discovered thus far is guaranteed to never decrease. 
Popular variants of the algorithm are 
\oneoneEA \cite[Algorithm~1]{Droste:2002}, 
$(1+\lambda)~\ea$ \cite[Algorithm 1]{Jansen2005} 
and $(\mu+1)~\ea$ \cite[Definition 1]{Witt2006}.
For comparison, we also consider 
the non-elitist $(\mu,\lambda)~\ea$, defined in
Algorithm~\ref{11-ea} \cite{LehreOliveto2017}, with 
mutation rate $\chi/n$ for any constant $\chi\in (0,n/2)$. In 
this algorithm, the
next population consists of the 
$\mu$ fittest individuals chosen from a set of  
$\lambda>\mu$ offspring produced by mutation.  

\begin{algorithm}
    \DontPrintSemicolon
    \caption{$(\mu,\lambda)~\ea$ with mutation rate $\chi/n$}\label{11-ea}
    $t\leftarrow 0$\;
    $P_t\leftarrow (x^{(1)},x^{(2)},\ldots,x^{(\mu)})$ 
    uniformly at random from $\mathcal{X}^\lambda$\; 
    \Repeat{termination condition is fulfilled}{            
        \For {$i=1,2,\ldots,\lambda$}{
            select $j\in [\mu]$ uniformly at random\;
           create $y^{(i)}$ by 
           flipping each bit in $x^{(j)}$ independently with probability
           $\chi/n$\;
        }    
        sort $(y^{(1)},\ldots,y^{(\lambda)})$ 
                such that $f(y^{(1)} ) \ge \ldots \ge f(y^{(\lambda)})$, 
                where ties are broken uniformly at random\;
        $P_{t+1}\leftarrow (y^{(1)},y^{(2)},\ldots,y^{(\mu)})$\;
        $t\leftarrow t+1$\;
}
\end{algorithm}

\subsection{Univariate marginal distribution algorithm}

The \umda, defined in Algorithm~\ref{umda-algor},
maintains a univariate model in each iteration $t\in \mathbb{N}$
that is 
represented as an $n$-vector 
$p_t\coloneqq  (p_{t,1},\ldots,p_{t,n})$, where 
each marginal (or frequency) $p_{t,i}$ 
for each $i \in [n]$  is
the probability of sampling a 1 at the $i$-th bit position of the
offspring. The probability of sampling
$x=(x_1,x_2,\ldots,x_n)$ from the model $p_t$ is 
$$
\Pr\left(x \mid p_t\right)
=\prod_{i=1}^{n}\left(p_{t,i}\right)^{x_i} \left(1-p_{t,i}\right)^{1-x_i}.
$$
The starting model is the uniform 
distribution $p_0\coloneqq  (1/2,\ldots,1/2)$. 
In each so-called iteration, the algorithm samples
a population $P_t$ of $\lambda$ individuals and sorts them
in descending order according to fitness. 
Let $X_{t,i}$ denote the number of 1s 
in bit position $i\in [n]$
among the $\mu$ fittest 
individuals. The marginals are updated 
using the component-wise formula:  
$
p_{t+1,i} \coloneqq   X_{t,i}/\mu
$
for all $i\in [n]$. 
Recall that  $\gamma^*=\mu/\lambda\in (0,1]$ is 
 the \textit{selective pressure} of the algorithm. 
The algorithm also restricts the marginals to be within
$[1/n,1-1/n]$ to avoid premature convergence.

\begin{algorithm}
    \DontPrintSemicolon        
        $t\leftarrow 0$\;
        initialise $p_t\leftarrow (1/2,1/2,\ldots,1/2)$\;
        \Repeat{termination condition is fulfilled}{
            \For{$j=1,2,\ldots,\lambda$}{
                sample $x_{t,i}^{(j)} \sim \Ber(p_{t,i})$ 
                for each $i\in [n]$ \;
            }
            sort $(x_t^{(1)},\ldots,x_t^{(\lambda)})$ such that
            $f(x^{(1)})\ge \ldots\ge f(x^{(\lambda)})$, where ties are broken uniformly
            at random\;
            \For{$i=1,2,\ldots,n$}{
                $X_{t,i}=\sum_{j=1}^\mu x_{t,i}^{(j)}$\;
                $p_{t+1,i} \leftarrow \max\{1/n, \min\{1-1/n, X_{t,i}/\mu\}\}$\;
              }
            $t\leftarrow t+1$\;
        }
    \caption{\umda with margins \label{umda-algor}}
\end{algorithm}

\subsection{Mutual-information-maximising input cluster algorithm}

The \mimic \cite{DeBonet:1996:MFO:2998981.2999041}
is a well-known bivariate \eda, which takes advantage of 
second-order statistics in an attempt to 
model the correlations between decision variables.
More specifically, let $p^*(X)$ denote the true distribution underlying the 
$\mu$ selected individuals, where
$X=(X_1,X_2,\ldots,X_n)$. It is often intractable to learn the 
true distribution $p^*(X)$, so 
an approximation to the
distribution is often preferred. Following this approach, 
the \mimic approximates $p^*(X)$ by a chain-structured model 
$\hat{p}(X)$, which is easy to learn and  sample. 
Given a permutation $\pi=(\pi_1,\pi_2,\ldots,\pi_n)$ 
of a set $[n]$, the chain-structured model is defined as follows:
\begin{displaymath}
\hat{p}_{\pi}(X) = p(X_{\pi_1})p(X_{\pi_2}\mid X_{\pi_1})\cdots p(X_{\pi_{n}}\mid X_{\pi_{n-1}}),
\end{displaymath}
where the $X_{\pi_1}$ is called the root of the chain.
The parameters of 
the model are then derived by 
minimising the Kullback-Leibler divergence  between models 
$\hat{p}_\pi(X)$ and $p^*(X)$, which
is equivalent to minimising an alternative cost 
function
\begin{displaymath}
J_{\pi}(X)=h(X_{\pi_1}) + h(X_{\pi_2}\mid X_{\pi_1}) +\cdots + h(X_{\pi_n}\mid X_{\pi_{n-1}}),
\end{displaymath}
where 
$h(X_{\pi_i}) =-\mathbb{E}[\log p(X_{\pi_i})]$ and 
$h(X_{\pi_i}\mid X_{\pi_j})=-\mathbb{E}[\log p(X_{\pi_i}\mid X_{\pi_j})]$
are the entropy and conditional entropy, 
respectively \cite{DeBonet:1996:MFO:2998981.2999041}.

The model is constructed in each iteration as follows. 
The variable with the smallest entropy is selected
to become the root of the chain. 
The variable among the remaining variables that has the smallest conditional entropy on
the root is then added to the root. 
This procedure is repeated until all variables are added to the 
chain (see steps 5--7). 
Sampling from the chain-structured 
model is even simpler. The root is sampled first using its marginal probability.
The next variable in the chain is then sampled using the 
conditional probability on the preceding variable. 
This is repeated until the end of the 
chain is reached (see steps 8--11). 
The procedure is known as \textit{ancestral sampling}
\cite{Bishop:2006:PRM}.
Denote the permutation in an iteration $t\in \mathbb{N}$ as
$\pi^{(t)}\coloneqq   \{\pi_1^{(t)},\pi_2^{(t)},\ldots,\pi_n^{(t)}\}$, where
$\pi_1^{(t)}$ denotes the root.

\begin{algorithm}[t]
    \DontPrintSemicolon        
        $t\leftarrow 0$ \;
        $P_t\leftarrow (x^{(1)},x^{(2)},\ldots,x^{(\lambda)})$ uniformly at random from $\mathcal{X}^\lambda$\;
        \Repeat{termination condition is fulfilled}{
            sort $P_t$ such that
            $f(x^{(1)})\ge f(x^{(2)})\ge \ldots\ge f(x^{(\lambda)})$\;
            $\pi_1^{(t+1)}\leftarrow \text{argmin}_{j\in [n]}~ h(X_j)$\;
            \For{$k=2,3,\ldots,n$}{
                $\pi_{k}^{(t+1)} \leftarrow \text{argmin}_j~ h(X_j\mid X_{\pi_{k-1}^{(t+1)}})$ 
                where $j\in [n]\setminus \{\pi_1^{(t+1)},\ldots,\pi_{k-1}^{(t+1)}\}$
            }             
            \For{$j=1,2,\ldots,\lambda$}{   
                       $x_{\pi_1^{(t+1)}}^{(j)}\leftarrow 
                    \begin{cases}
                          1   &   \text{w.p.} \quad \text{R}(p(X_{\pi_1^{(t+1)}})),    \\
                          0 &   \text{w.p.} \quad 1-\text{R}(p(X_{\pi_1^{(t+1)}}))   
                    \end{cases}$\;
               \For{$k=2,3,\ldots,n$}{
                           $x_{\pi_k^{(t+1)}}^{(j)}\leftarrow 
                        \begin{cases}
                        1 & \text{w.p.} \quad 
                                \text{R}(p(X_{\pi_{k}^{(t+1)}}\mid X_{\pi_{k-1}^{(t+1)}})),\\
                        0 & \text{w.p.} \quad 
                                1-\text{R}(p(X_{\pi_{k}^{(t+1)}}\mid X_{\pi_{k-1}^{(t+1)}}))
                        \end{cases}$\;                        
                    }                
            }
            $P_{t+1}\leftarrow (x^{(1)}, x^{(2)}, \ldots,x^{(\lambda)})$\;
            $t\leftarrow t+1$\;
        }
    \caption{\mimic with margins, where ties occurring in 
    sorting and selection are broken uniformly at random.}
    \label{mimic-algor}
\end{algorithm}

We note that the description of the \mimic algorithm in
\cite{DeBonet:1996:MFO:2998981.2999041} is very ambiguous. 
In particular, when constructing the model, we need to calculate
many conditional entropies, which in turn require the 
calculation of many conditional probabilities. We can take as an 
example the conditional probability 
$p(X_{\pi_i}=\alpha\mid X_{\pi_j}=\beta)$ for $\alpha, \beta\in \{0,1\}$, 
which by the definition can be written as 
$$
p(X_{\pi_i}=\alpha\mid X_{\pi_j}=\beta)
=p(X_{\pi_i}=\alpha, X_{\pi_j}=\beta)/p(X_{\pi_j}=\beta).
$$
Note that all probabilities on the right-hand side 
will be directly measured from the sampled 
population (via counting). And, this is where the problem comes from 
since very often that the event $\{X_{\pi_j}=\beta\}$ may not happen, leading
to the probability estimate $p(X_{\pi_j}=\beta)\approx 0$ which renders the
conditional probability above undefined.
To avoid this problem, we make use of a function $R(p)\coloneqq  \max\{p,1/n\}$ 
to ensure that the probability $p(X_{\pi_j}=\beta)$ is always
positive, and never less than $1/n$. Note that this bound is in
line with the way marginal probabilities are bounded in other EDAs,
such as the \umda. We think that this small 
change is essential to ensure that
the pseudo-code of \mimic is well-defined. 
Algorithm~\ref{mimic-algor} gives a full description of the \mimic (with margins)
\cite[Chapter~13]{simon2013evolutionary}.

\subsection{Level-based analysis} 
\label{subsec:level-based-theorem}

First proposed in \cite{Lehre:2011:FNP}, 
the level-based theorem is a general tool that provides 
upper bounds on the expected runtime of 
many non-elitist population-based algorithms on a wide range of 
optimisation problems 
\cite{bib:Corus2016,Doerr2019-MulUpTheorem, Dang:2015,bib:lehre2018a,bib:Lehre2017,Lehre2018PBIL,bib:Lehre2019}.
The theorem assumes that the studied algorithm 
can be described in the form of 
Algorithm \ref{abstract-algor}, which never assumes
specific fitness functions, selection mechanisms, or
generic operators like mutation and crossover. 
The search space $\mathcal{X}$ is 
partitioned into $m$ disjoint subsets 
$A_1,\ldots,A_m$, which we call \textit{levels}, and 
the last level $A_m$ 
consists of global optima of the objective function.
Denote $A_{\ge j}\coloneqq  \cup_{k=j}^m A_k$. 

\begin{algorithm}[t]
    \DontPrintSemicolon
    $t \leftarrow 0$; create initial population $P_t$\;
    \Repeat{termination condition is fulfilled}{
      \For{$i=1,\ldots,\lambda$}{
                sample $P_{t+1,i} \sim \mathcal{D}(P_t)$\;}
                $t \leftarrow t+1$\;
            }
    \caption{Non-elitist population-based algorithm\label{abstract-algor}}
\end{algorithm}

\begin{theorem}[\cite{bib:Corus2016}]
\label{thm:levelbasedtheorem}
    Given a partition $\left(A_i\right)_{i \in [m]}$ of $\sspace$, define 
    \begin{math}
        T\coloneqq  \min\{t\lambda \mid |P_t\cap A_m|>0\},
              \end{math}
              where for all $t\in\Natural$,
              $P_t\in\sspace^\lambda$ is the population of
              Algorithm~\ref{abstract-algor} in iteration $t$.
              Denote $y \sim \mathcal{D}(P_t)$. If there exist 
    $z_1,\ldots,z_{m-1}, \delta \in (0,1]$, and $\gamma_0\in (0,1)$ 
    such that for any population $P_t \in \sspace^\lambda$,
    \begin{itemize}
        \item[(G1)] for each level 
        $j\in[m-1]$, 
        if $|P_t\cap A_{\geq j}|\geq \gamma_0\lambda$ then 
            $$
                \Pr\left(y \in A_{\geq j+1}\right) \geq z_j,
                $$
\item[(G2)] for each level $j\in[m-2]$ and all 
                $\gamma \in (0,\gamma_0]$, if $|P_t\cap A_{\geq j}|\geq     
            \gamma_0\lambda$ and $|P_t\cap A_{\geq j+1}|\geq \gamma\lambda$ then
            $$
                \Pr\left(y \in A_{\geq j+1}\right) \geq \left(1+\delta\right)\gamma,
            $$
        \item[(G3)] and the population size $\lambda \in \Natural$ satisfies
            $$
             \lambda \geq \left(\frac{4}{\gamma_0\delta^2}\right)\ln\left(\frac{128m}{z_*\delta^2}\right), 
        $$
            where $z_* \coloneqq   \min_{j\in [m-1]}\{z_j\}$, then
            \begin{displaymath}
                \mathbb{E}\left[T\right] \leq \left(\frac{8}{\delta^2}\right)\sum_{j=1}^{m-                                                            1}\left[\lambda\ln\left(\frac{6\delta\lambda}{4+z_j\delta\lambda}\right)+\frac{1}{z_j}\right].
            \end{displaymath}
    \end{itemize}
\end{theorem}

\subsection{Useful Tools from Probability Theory}

We will use the following well-known tail bounds
\cite{bib:Motwani1995, Dubhashi:2009:CMA}.

\begin{lemma}[Chernoff Bound]
\label{lemma:chernoff-bound}
Let $b>0$.
Let $Y_1,\ldots,Y_n$ be independent random variables
(not necessarily i.i.d.), that take values in $[0,b]$.
Let $Y\coloneqq  \sum_{i=1}^n Y_i$, and $\mu\coloneqq  \mathbb{E}[Y]$. Then
for any $0\le \delta \le 1$,
$$
\Pr(Y\le (1-\delta)\mu)\le e^{-\delta^2 \mu/(2b)},
$$
and 
$$
\Pr(Y\ge (1+\delta)\mu)\le e^{-\delta^2\mu/(3b)}.
$$
\end{lemma}

\begin{lemma}[Chernoff-Hoeffding Bound]
\label{lemma:chernoff-hoeffding}
Let $Y_1,\ldots,Y_n$ be independent random variables
(not necessarily i.i.d.), where 
$Y_i$ takes values in $[0,b_i]$. Let 
$Y\coloneqq  \sum_{i=1}^n Y_i$, and let $b\coloneqq  \sum_{i=1}^n b_i^2$. Then 
$\Pr(|Y-\mathbb{E}[Y]|\ge t)\le 2e^{-2t^2/b}$.
\end{lemma}

\begin{lemma}[\cite{Second-moment}]
\label{lemma:second-moment}
$\mathbb{E}[X^2\mid X\sim \bin{n,p}]= n p (p(n - 1) + 1)$.
\end{lemma}

\section{EAs optimise DLB efficiently}
\label{sec:eas-on-dlb}

We start by showing that simple \eas optimise the
\dlblock function  in polynomial expected 
runtime. We consider both elitist and non-elitist
\eas, namely 
$(1+\lambda)~\ea$,
$(\mu+1)~\ea$ and $(\mu,\lambda)~\ea$, in addition to
Genetic Algorithms.  
Although these \eas
are simple, we analyse them 
here to emphasise their efficiency in dealing with epistasis and 
mild deception. Speaking of proving techniques, 
we will  use the fitness-level method \cite{Wegener2002} and the
level-based theorem (see Theorem~\ref{thm:levelbasedtheorem}).
In doing so, the search space $\mathcal{X}$ is first
partitioned into non-empty disjoint subsets
$A_0,A_1,\ldots,A_m$ (called \textit{levels}, 
where $m\coloneqq  n/2$) such that  
\begin{equation}\label{level-definition}
A_i=\{x\in \mathcal{X} : \phi(x) = i \},
\end{equation}
where $\phi(x)$ is defined in (\ref{eq:leading-blocks}),
and $A_m$ contains the all-ones bitstring. 
We now give runtime bounds on the \dlblock function 
for the \eas; the proofs are straightforward.

\begin{theorem}\label{thm:one-lambda-ea}
The expected runtime of the $(1+\lambda)$~\ea 
on the \dlblock function is $\bigO{\lambda n+n^3}$.
\end{theorem}

\begin{proof}
Levels are defined as in (\ref{level-definition}). 
The probability of leaving the current level $i<m$ 
is lower bounded by $(1-1/n)^{n-2}(1/n)^2\ge 1/en^2$, 
and thus not leaving it happens with probability at most $1-1/en^2$. 
In each iteration, the $(1+\lambda)$~\ea samples $\lambda$ individuals by mutating the current bitstring. 
At least one among 
$\lambda$ individuals leaves the current level 
with probability at least 
$
1-(1-1/en^2)^{\lambda} \ge 1-e^{-\lambda/en^2}.
$
Note that if $\lambda \ge en^2$, then 
this probability is at least $1-1/e$; 
otherwise, it is at least $\lambda/2en^2$. 
Putting everything together, the 
expected runtime guaranteed 
by the fitness-level method is 
\begin{displaymath}
\lambda\cdot \sum_{i=0}^{n/2-1}\left(\bigO{1} + \frac{2en^2}{\lambda}\right)
= \bigO{n\lambda +n^3}.\qedhere
\end{displaymath}
 \end{proof}

\begin{theorem}\label{thm:mu-one-ea}
The expected runtime of the $(\mu+1)$~\ea on the \dlblock function is 
$\bigO{\mu n\log n+n^3}$.
\end{theorem}

\begin{proof}
Levels are defined as in (\ref{level-definition}).  
It suffices to correct block $i+1$ to leave the
current level. Following \cite{Witt2006}, 
we define a fraction $\chi(i)\coloneqq  n/\log n$.
Given $j$ copies of the best individual, another one is created with
probability $(j/\mu)(1-1/n)^n \ge j/2e\mu$. Thus, the expected time 
for a fraction $\chi(i)$ of the population 
to be in level $i$ is given by 
$$
T_0 \le 2e\mu\sum_{j=1}^{n/\log n} (1/j) \le 2e\mu \log n.
$$

Now given $\chi(i)$ individuals
in level $i$, the event of leaving 
this level occurs with probability 
$
s_i \ge  (\chi(i)/\mu)(1-1/n)^{2i} (1/n)^2 \ge (\chi(i)/\mu)\cdot 1/en^2.
$ 
This probability is at 
least 
\begin{displaymath}
s_i = \begin{cases}
1/en^2, &\text{if } \mu \le \chi(i)\\
1/(e\mu\log n), &\text{if } \mu>\chi(i).
\end{cases}
\end{displaymath}
The expected runtime of the algorithm 
on \dlblock is 
\begin{displaymath}
    \sum_{i=0}^{n/2-1} \left(T_0 + \frac{1}{s_i}\right)
= \bigO{\mu n\log n+n^3}. \qedhere
\end{displaymath}
\end{proof}

\begin{theorem}\label{thm:mu-lambda-ea}
The expected runtime of the $(\mu,\lambda)$~\ea 
with $\lambda\ge c\log n$ for some 
sufficiently large constant $c>0$ and
$\mu\le \lambda e^{-2\chi}/(1+\delta)$ for  any constant 
$\delta>0$ and a mutation rate constant 
$\chi \in (0,n/2)$ on the \dlblock function
is $\bigO{n\lambda\log \lambda +n^3}$.
\end{theorem}

\begin{proof}
Since $(\mu,\lambda)~\ea$ is non-elitist, 
 Theorem~\ref{thm:levelbasedtheorem}
  guarantees an upper bound 
 on the expected runtime as long as the three 
conditions (G1), (G2) and (G3) are fully verified. 
Choose $\gamma_0\coloneqq  \mu/\lambda$. 
The levels are defined as in (\ref{level-definition}),
and $A_j$ is assumed to be the current level. 

Condition (G1) requires a lower bound on
the probability of sampling an offspring
in $A_{\ge j+1}$, where $A_{\ge j+1}\coloneqq   \cup_{k=j+1}^m A_{k}$,
given $|P_t\cap A_{\ge j}|\ge \gamma_0\lambda=\mu$. 
During the selection step, 
if we choose an individual in $A_{\ge j}$, then the step 
is successful if the mutation operator correctly flips 
two of the bits in the active block while keeping others unchanged.
Thus, the probability of a successful sampling is at least
$$
(1-\chi/n)^{n-2}(\chi/n)^2\ge e^{-\chi}\chi^2/n^2.
$$

Next, condition (G2) assumes that
at least $\gamma\lambda <\mu$ individuals 
have at least $j+1$ leading 11s.
It suffices to pick one of the $\gamma\lambda$ 
fittest individuals and flip none of the bits;
the probability is at least 
$$
(\gamma\lambda/\mu)(1-\chi/n)^{n} 
\ge (\gamma/\gamma_0)e^{-2\chi} \ge (1+\delta)\gamma
$$
if $\gamma_0\le e^{-2\chi}/(1+\delta)$ 
for any constant $\delta>0$.

Putting everything into condition (G3) yields 
$\lambda\ge c\log (n)$ for a sufficiently 
large constant $c>0$. 
Having fully verified three conditions,
the expected runtime of 
the $(\mu,\lambda)~\ea$ on \dlblock is 
\begin{displaymath}
\mathcal{O}\left(\sum_{i=0}^{n/2-1}\left(\lambda\log \lambda + n^2\right)\right) 
= \mathcal{O}\left(n\lambda\log \lambda + n^3\right).\qedhere
\end{displaymath}
\end{proof}

So far we have considered only mutation-based \eas,  
the following theorem shows that Genetic Algorithms, defined in 
\cite[Algorithm~2]{bib:Corus2016}, with
crossover rate $p_c$ using any crossover operator also take
a polynomial expected runtime to 
optimise the \dlblock function. 

\begin{theorem}
Genetic Algorithms with crossover rate 
$p_c = 1- \Omega(1)$ using any crossover operator, 
the bitwise mutation operator with mutation rate $\chi/n$ for
any fixed constant $\chi>0$ and one of the following 
selection mechanisms: $k$-tournament
selection, $(\mu,\lambda)$-selection, 
linear or exponential ranking selection, with their parameters
$k$, $\lambda/\mu$ and $\eta$ being set to no less than 
$(1 + \delta)e^\chi/(1- p_c)$ where $\delta \in (0,1]$
being any constant, take an  
$\bigO{n^3+n\lambda\log \lambda}$  expected runtime 
 on the \dlblock function, where $\lambda \ge c\log n$ for some sufficiently large constant 
 $c>0$.
\end{theorem}

\begin{proof} The results for $k$-tournament, $(\mu,\lambda)$-selection and linear ranking
follow by applying \cite[Lemmas~5--7]{Lehre:2011:FNP}, while the result for 
exponential ranking can be seen in \cite[Lemma~3]{bib:Corus2016}.
\end{proof}

\section{Why is UMDA inefficient on DLB?}
\label{sec:umda-on-dlb-low}

Before we get to analysing the \umda on the \dlblock function, we introduce 
some notation. Recall that there are $m\coloneqq  n/2$ blocks. We then let
$C_{t,i}$ for each $i\in [m]$ 
denote the number of individuals
having at least $i$ leading 11s
in iteration $t\in \mathbb{N}$, and 
$D_{t,i}$ denote the number of individuals having 
$i-1$ leading 11s, 
followed by a 00 at the $i$-th block. For the special case of $i=1$,
$D_{t,1}$ consists of those with the first block being a 00.
We also let $E_{t,i}$ 
denote the number of individuals having  
$i-1$ leading 11s, followed by a 10 at 
the $i$-th block, and again
$E_{t,1}$ consists of those having the 
first block being a 10.

Once the population has been sampled, 
the algorithm invokes truncation selection to 
select the $\mu$ fittest individuals to update the probability vector. 
We take this $\mu$-cutoff into account
by defining a random variable 
\begin{equation}
\label{eq:Z-definition}
    Z_t \coloneqq   \max\{i\in \mathbb{N}\cap [0,m] : C_{t,i}\ge \mu\},
\end{equation}
which tells us 
how many consecutive marginals, counting from position one, 
are set to the upper border $1-1/n$
in iteration $t$. We also define another random variable
\begin{equation}
    \label{eq:Z-star-definition}
    Z_t^*\coloneqq  \max\{i\in \mathbb{N}\cap [0,m] : C_{t,i}>0\}
\end{equation}
to be the number of leading 11s
of the fittest individual(s). For readability, we often leave out the 
indices of random variables like when we write $C_{t}$ instead of $C_{t,i}$, 
if values of the indices are clear from the context. 
Furthermore, let $(\mathcal{F}_t)_{t\in \mathbb{N}}$
be a filtration induced from the 
population $(P_t)_{t\in \mathbb{N}}$, and we often write 
$\mathbb{E}_t[X]\coloneqq  \mathbb{E}[X\mid \mathcal{F}_t]$
and $\var_t[X]\coloneqq  \var[X\mid \mathcal{F}_t]$.

\subsection{On the distributions of $C_{t,i}, D_{t,i}$ and $E_{t,i}$}

We apply the \textit{principle of deferred
decisions} \cite{bib:Motwani1995} and imagine that the algorithm first 
samples the values of the first block for  $\lambda$
individuals. Once this is finished, it moves on to
the second block and so on 
until the whole population is obtained.

We note that  \textit{selection} 
prefers individuals 
with the first block being a 11 to those with a 00, which in turn 
is more preferred to those with a 10 or 01 
(due to deceptive fitness). 
The number of 11s in the first block 
follows a binomial distribution with parameters 
$\lambda$ and $p_{t,1}p_{t,2}$, that is,
$C_{t,1}\sim \bin{\lambda,p_{t,1}p_{t,2}}$.
Having sampled $C_{t,1}$ 11s, there are  
$\lambda-C_{t,1}$ other blocks in block $1$ 
in the current population. $D_{t,1}$ 
is also binomially distributed with parameters $\lambda-C_{t,1}$ and 
$(1-p_{t,1})(1-p_{t,2})/(1-p_{t,1}p_{t,2})$ by the definition of conditional
probability since the event of sampling a 11 is excluded.
Similarly having sampled 11s and 00s, 
$E_{t,1}$ is binomially distributed with $\lambda-C_{t,1}-D_{t,1}$ trials and 
success probability 
$(p_{t,1}(1-p_{t,2}))/(1-p_{t,1}p_{t,2}-(1-p_{t,1})(1-p_{t,2}))$ 
since again the event
of sampling either a 11 or a 00 is excluded.
Finally, the number of 01s is 
$\lambda-C_{t,1}-D_{t,1}-E_{t,1}$. 

Having sampled the first block for $\lambda$ individuals, 
and note that the \textit{bias} due to selection in the 
second block comes into play only if 
the first block is a 11. 
Among the $C_{t,1}$ fittest individuals, those with
a 11 in the second block
will be ranked first, followed by those with a 00, 
and finally with a 10 or 01. 
Conditioned on the first block being a 11, 
the number of 11s in the second block
is binomially distributed with parameters 
$C_{t,1}$ and $p_{t,3}p_{t,4}$, i.e.,
$C_{t,2} \sim \bin{C_{t,1},p_{t,3}p_{t,4}}$,
and the number of 
00s also follows a binomial distribution
with $C_{t,1}-C_{t,2}$ trials and success probability 
$(1-p_{t,3})(1-p_{t,4})/(1-p_{t,3}p_{t,4})$. Similarly,
$E_{t,2}$ is binomially distributed 
with parameters $C_{t,1}-C_{t,2}-D_{t,2}$ and
$p_{t,3}(1-p_{t,4})/(1-p_{t,3}p_{t,4}-(1-p_{t,3})(1-p_{t,4}))$, 
and finally the number of 01s equals  
$C_{t,1}-C_{t,2}-D_{t,2}-E_{t,2}$.
Unlike the first block, we also have
$\lambda-C_{t,1}$ remaining individuals, and
since there is \textit{no bias} in the second block 
among these individuals, 
the numbers of 1s sampled at 
the two bit positions are binomially distributed with 
$\lambda-C_{t,1}$ trials and success probabilities 
$p_{t,3}$ and $p_{t,4}$, respectively. 

We now consider an arbitrary block $i\in [m]$.
By induction, we observe that
the number of individuals having at least $i$ leading 11s
follows a binomial distribution with parameters $C_{t,i-1}$ and 
$p_{t,2i-1}p_{t,2i}$, that is,
\begin{equation}
\label{eq:C-distribution}
    C_{t,i} \sim \bin{C_{t,i-1}, p_{t,2i-1}p_{t,2i}}.
\end{equation}
Similarly,
\begin{equation}
\label{eq:D-distribution}
D_{t,i} \sim \bin{C_{t,i-1}-C_{t,i},\frac{(1-p_{t,2i-1})(1-p_{t,2i})}{1-p_{t,2i-1}p_{t,2i}}},
\end{equation}
and
\begin{equation}
    \label{eq:E-distribution}
 E_{i,t}\sim 
 \bin{C_{t,i-1}-C_{t,i}-D_{t,i},\frac{p_{t,2i-1}(1-p_{t,2i})}{p_{t,2i-1}+p_{t,2i}-2p_{t,2i-1}p_{t,2i}}}.
\end{equation}
Finally, the number of individuals with $i-1$ leading 11s followed 
by a 01 in the block $i$ 
is $C_{t,i-1}-C_{t,i}-D_{t,i}-E_{t,i}$.
For the $\lambda-C_{t,i-1}$ remaining individuals, 
the numbers of 1s sampled in the 
bit positions $2i-1$ and $2i$ follow 
a $\bin{\lambda-C_{t,i-1},p_{t,2i-1}}$
and a $\bin{\lambda-C_{t,i-1},p_{t,2i}}$, respectively.
We note in particular 
that by the end of this alternative view on the sampling process, 
we obtain the population of 
$\lambda$ individuals sorted in descending order according to fitness, where
ties are broken uniformly at random.
The following lemma provides
the expectations of these random variables.

\begin{lemma}
For all  
$t\in \mathbb{N}$, if $i=1$ then
\begin{align*}
  \mathbb{E}_{t-1}[C_{t,i}]
  &= \lambda\cdot p_{t,2i-1}p_{t,2i}, \\    
  \mathbb{E}_{t-1}[D_{t,i}] 
  &= \lambda\cdot (1-p_{t,2i-1})(1-p_{t,2i}),\\
  \mathbb{E}_{t-1}[E_{t,i}]
  &=\lambda\cdot p_{t,2i-1}(1-p_{t,2i}).
  \end{align*}
Otherwise, if $i\in [m]\setminus \{1\}$, then
\begin{align}
  \begin{split}\label{eq:expectation-of-C}
  \mathbb{E}_{t-1}[C_{t,i}]
  &= \mathbb{E}_{t-1}[C_{t,i-1}]\cdot p_{t,2i-1}p_{t,2i}, \\    
  \end{split}\\
  \begin{split}\label{eq:expectation-of-D}
  \mathbb{E}_{t-1}[D_{t,i}] 
  &= \mathbb{E}_{t-1}[C_{t,i-1}]\cdot (1-p_{t,2i-1})(1-p_{t,2i}),\\
  \end{split}   \\
  \begin{split}\label{eq:expectation-of-E}
  \mathbb{E}_{t-1}[E_{t,i}]
  &=\mathbb{E}_{t-1}[C_{t,i-1}]\cdot p_{t,2i-1}(1-p_{t,2i}).\\
  \end{split}
  \end{align}
\end{lemma}

\begin{proof}
For the special case of $i=1$, 
the expectations are trivial since 
random variables 
$C_{t,i}, D_{t,i}$ and $E_{t,i}$ are all binomially 
distributed with $\lambda$ trials. In the remainder of the proof, 
we will consider the case of $i\neq 1$.
By (\ref{eq:C-distribution}), 
 the tower rule $\mathbb{E}[X]=\mathbb{E}[\mathbb{E}[X\mid Y]]$ 
 \cite{bib:Feller1968} and noting that 
 $p_t$ is $\mathcal{F}_{t-1}$-measurable, 
 we  get
 \begin{align*}
           \mathbb{E}_{t-1}[C_{t,i}] 
    &= \mathbb{E}_{t-1}[\mathbb{E}_{t-1}[C_{t,i}\mid C_{t,i-1}]] \\
    &= \mathbb{E}_{t-1}[\mathbb{E}_{t-1}[\bin{C_{t,i-1},p_{t,2i-1}p_{t,2i}}\mid C_{t,i-1}]]\\
    &=\mathbb{E}_{t-1}[C_{t,i-1}\cdot p_{t,2i-1}p_{t,2i}]\\
    &=\mathbb{E}_{t-1}[C_{t,i-1}]\cdot p_{t,2i-1}p_{t,2i}.
 \end{align*}
By (\ref{eq:D-distribution}) and (\ref{eq:expectation-of-C}), we also get
\begin{align*}
        \mathbb{E}_{t-1}[D_{t,i}]
        &= \mathbb{E}_{t-1}[\mathbb{E}_{t-1}[D_{t,i}\mid C_{t,i-1}, C_{t,i}]]\\
        &=\mathbb{E}_{t-1}\left[(C_{t,i-1}-C_{t,i})
            \frac{(1-p_{t,2i-1})(1-p_{t,2i})}{1-p_{t,2i-1}p_{t,2i}}\right]\\
        &=\mathbb{E}_{t-1}[C_{t,i-1}](1-p_{t,2i-1}p_{t,2i})
            \frac{(1-p_{t,2i-1})(1-p_{t,2i})}{1-p_{t,2i-1}p_{t,2i}}\\
         &= \mathbb{E}_{t-1}[C_{t,i-1}]\cdot (1-p_{t,2i-1})(1-p_{t,2i}),
\end{align*}
and similarly by (\ref{eq:E-distribution}), (\ref{eq:expectation-of-C}) and 
(\ref{eq:expectation-of-D}), we finally obtain 
  \begin{displaymath}
   \mathbb{E}_{t-1}[E_{t,i}]
  =\mathbb{E}_{t-1}[C_{t,i-1}]\cdot p_{t,2i-1}(1-p_{t,2i}).\qedhere
  \end{displaymath}
\end{proof}

\subsection{In the initial population}

An initial observation is that
the all-ones bitstring cannot be sampled in the initial population $P_0$ 
with high probability since 
the probability of sampling 
it from the uniform distribution 
is $2^{-n}$, then by the union bound \cite{bib:Motwani1995}
it appears in the initial population of $\lambda$ individuals
with probability at most 
$\lambda\cdot 2^{-n} =2^{-\Omega(n)}$ since we 
only consider the offspring population of 
size at most polynomial in the problem instance size $n$.
The following 
lemma states the expectations of
the random variables $Z_t^*$ and $Z_t$ (defined in (\ref{eq:Z-definition}) and 
(\ref{eq:Z-star-definition}), respectively) 
in the iteration $t=0$.

\begin{lemma}
\label{lemma:fittest-individual}
    $\mathbb{E}[Z_0^*]=\bigO{\log \lambda}$, and
    $\mathbb{E}[Z_0] = \bigO{\log (\lambda-\mu)}$.
\end{lemma}

\begin{proof}
Recall that $Z_0^* =\max\{i : C_{0,i} >0\}$ 
and the definition of the function $\phi(x)$  in 
(\ref{eq:leading-blocks}). The probability of sampling 
    an individual with $k$ leading 11s (where $k<m$) 
    is 
    $
    \Pr(\phi(x)=k)=(1/4)^k (1-1/4)=3\cdot 4^{-(k+1)},
    $ then
    $\Pr(\phi(x)\le k) = 1-4^{-(k+1)}$. 
    The event $\{Z_0^*\le k\}$ implies 
    that the $\lambda$
    individuals all have at most $k$ leading 11s, i.e.,
    $$
    \Pr(Z_0^* \le k)
    =\prod_{i=1}^{\lambda}\Pr(\phi(x_0^{(i)})\le k)
    =(1-4^{-(k+1)})^{\lambda},
    $$
    and 
    $\Pr(Z_0^*>k)= 1-(1-4^{-(k+1)})^{\lambda}$. 
    Since the random variable $Z_0^*$ is integer-valued and by 
    $\sum_{i=1}^k (1/i)<1+\ln k$, we get
    \begin{align*}
    \mathbb{E}\left[Z_0^*\right] 
    &<\sum_{k=0}^{\infty} \Pr(Z_0^*>k)\\
    &= \sum_{k=0}^{\infty} (1-(1-4^{-(k+1)})^{\lambda})\\
    &<1+\int_{0}^{\infty} (1-(1-e^{-x\ln 4})^{\lambda})\text{d}x\\
    &= \frac{1}{\ln 4}\int_{0}^{1}\frac{1-u^\lambda}{1-u}du \quad\quad \text{(by following \cite{EISENBERG2008135})}\\
&=\frac{1}{\ln 4}\int_{0}^{1} \sum_{i=0}^{\lambda-1}u^idu 
= \frac{1}{\ln 4} \sum_{i=1}^\lambda \frac{1}{i}
<\frac{1+\ln \lambda}{\ln 4}
=\bigO{\log \lambda},
    \end{align*}
which proves the first claim.
    
    For the second claim, we take an alternative view that
    the random variable $Z_0$ denotes the number of leading 11s of the fittest
    individual(s) in a smaller population of 
    the $\lambda-\mu+1$ remaining individuals 
    (i.e., all but the $\mu-1$ fittest individuals 
    in the initial population), 
    sampled from a uniform distribution. The same line of arguments above
    immediately yields 
    \begin{displaymath}
    \mathbb{E}\left[Z_0\right]
    <\frac{1+\ln (\lambda-\mu+1)}{\ln 4}
    =\bigO{\log (\lambda-\mu)}.\qedhere
    \end{displaymath}
\end{proof}

\subsection{In an arbitrary iteration $t$}

By the definition of the random variable $Z_t$, 
the first $2Z_{t}$ marginals are set to 
the upper border $1-1/n$ in iteration $t\in \mathbb{N}$.
Recall that the random variable  $X_{t,i}$ denotes the 
number of 1s in bit position $i\in [n]$ among the 
$\mu$ fittest individuals, which is used to update the probabilistic model of the \umda.
We also define another
random variable  $Y_{t,j}$ to be the number of 11s sampled in block
position $j\in [m]$, also 
among the $\mu$ fittest individuals in 
an iteration $t\in \mathbb{N}$.

\begin{lemma}
\label{lemma:distributions-X-Y-f-indpdent}
For any $t\in \mathbb{N}$ that 
\begin{itemize}
    \item[(a)] $Y_{t,j}\sim \bin{\mu,p_{t,2j-1}p_{t,2j}}$ for all $j\ge Z_t+2$, and
    \item[(b)] $X_{t,i}\sim \bin{\mu,p_{t,i}}$ for all $i\ge 2Z_t +3$.
\end{itemize}
\end{lemma}

\begin{proof}
By the definition of the random variable  $Z_t$, we know that $C_{t,Z_t}\ge \mu$ and 
$C_{t,Z_t+1}<\mu$. Consider the block $j\coloneqq Z_t+2$. We then obtain
from (\ref{eq:C-distribution}) that $C_{t,j}\sim \bin{C_{t,j-1}, p_{t,2j-1}p_{t,2j}}$.
For 
the $\mu-C_{t,j-1} >0$ remaining individuals (among the $\mu$ fittest individuals), these individuals have 
the block $j-1$ set to $\{00,10,01\}$. This means 
that the overall fitness (or the fitness ranking) of these individuals
have been already decided by the first $j-1$ blocks, and what 
is sampled in the block $j$ will not have 
any impact on the ranking of these individuals. Therefore, there is no bias in  block $j$ among 
these individuals, which literally means that the number of 11s sampled
here follows a binomial distribution with $\mu-C_{t,j-1}$ trials and success 
probability $p_{t,2j-1}p_{t,2j}$, 
i.e., $\bin{\mu-C_{t,j-1}, p_{t,2j-1}p_{t,2j}}$. 
Putting things together, the total number of 11s
sampled in the block $j$ among the $\mu$ fittest individuals 
equals 
\begin{align*}
    Y_{t,j}
    &\sim C_{t,j} + \bin{\mu-C_{t,j-1}, p_{t,2j-1}p_{t,2j}}\\ 
&\sim  \bin{C_{t,j-1}, p_{t,2j-1}p_{t,2j}} + \bin{\mu-C_{t,j-1}, p_{t,2j-1}p_{t,2j}}\\ 
&\sim \bin{\mu, p_{t,2j-1}p_{t,2j}}.
\end{align*}
We note that should 
this result hold for any block 
$j\ge Z_t+2$, which proves the first statement.

For the second statement, we consider a bit position $i=2j-1$ in 
block $j=Z_t+2$. We note that the number of 1s sampled in  bit
position $i$ can be written 
as the sum of three components: 
\begin{itemize}
    \item[(1)] the number of individuals with at least $j$ leading 11s 
    (i.e., $C_{t,j}$),
    \item[(2)] the number of individuals with $j-1$ leading 11s, 
    followed by a 10 block in  block $j$ (i.e., $E_{t,j}$), and
    \item[(3)] the number of 1s sampled in bit position $i$ among all the 
    $\mu$ fittest individuals except the top $C_{t,j-1}$ individuals. 
    There is 
    no bias among these individuals, so the number of 1s here is 
    binomially distributed with parameters
    $\mu-C_{t,j-1}$ and $p_{t,i}$.
\end{itemize}
We note further that the sum of 
$C_{t,j}+E_{t,j}$ equals the number of 1s sampled in bit position 
$i$ among the $C_{t,j-1}$ fittest individuals. Thus, we get:
\begin{align*}
    X_{t,i} 
    &\sim C_{t,j} + E_{t,j} + \bin{\mu-C_{t,j-1},p_{t,i}}\\
    &\sim \bin{C_{t,j-1},p_{t,i}} + \bin{\mu-C_{t,j-1},p_{t,i}}\\
    &\sim \bin{\mu, p_{t,i}}.
\end{align*}
By the same line of argumentation, we can show that 
the number of 1s sampled in  bit position $i+1$ is 
$X_{t,i+1}\sim \bin{\mu,p_{t,i+1}}$, and similarly for other
bit positions 
from $i+3$ to $n$.
\end{proof}

We now consider the block $i=Z_t+1$, where
$C_{t,i}<\mu$. The following lemma 
shows that if the value of the random variable $C_{t,i}$
is below a threshold, then in iteration $t+1$
the number of individuals with at least 
$i$ leading 11s sampled decreases, while the number of individuals with exactly 
$i-1$ leading 11s followed by a 00 increases in expectation. 
We consider two different
regimes of the selective pressure, i.e., 
$\gamma^*<1/2e$  and $\gamma^*\ge 1/2e$.

\begin{lemma}
\label{lemma:selective-pressure-less-1-2} 
Consider the block $i=Z_t+1$ in an 
arbitrary iteration $t\in \mathbb{N}$, and 
assume further that 
$C_{t,i}+D_{t,i}\ge \mu$.
\begin{itemize}
    \item[(A)] Let $\gamma^*=\mu/\lambda<1/2e$.
    If there exists a constant $\varepsilon\in (0,1)$ such that
 $C_{t,i}\le (\mu^2/\lambda)(1-\varepsilon)$,
then
\begin{itemize}
    \item[A.1)] $\mathbb{E}_t[C_{t+1,i}]
    < C_{t,i}(1-\varepsilon)$,  
    \item[A.2)] $\mathbb{E}_t[C_{t+1,i}+D_{t+1,i}]
    >\mu(1+\varepsilon^2)$, 
    \item[A.3)] $\Pr(C_{t+1,i}+D_{t+1,i}\le \mu)\le e^{-\Omega(\mu)}$, and
    \item[A.4)] $\Pr(C_{t+1,i}
    \ge (\mu^2/\lambda)(1-\varepsilon)) \le e^{-\Omega(\mu^2/\lambda)}$.
\end{itemize}

\item[(B)] Let $\gamma^*=\mu/\lambda\ge 1/2e$.
    If there exists a constant $\varepsilon\in (0,1)$ such that
 $C_{t,i} \le(\mu/2e) (1-\sqrt{\alpha})$, where $\alpha\coloneqq 2e(1+\varepsilon)(\mu/\lambda)-1\ge \varepsilon$,
    then
    \begin{itemize}
        \item[B.1)] $\mathbb{E}_t[C_{t+1,i}]<C_{t,i}(1-\sqrt{\varepsilon})$, 
        \item[B.2)] $\mathbb{E}_t[C_{t+1,i}+D_{t+1,i}]
      >\mu(1+\varepsilon)$, 
      \item[B.3)] $\Pr(C_{t+1,i} +D_{t+1,i}\le \mu) \le e^{-\Omega(\mu)}$, and 
      \item[B.4)] $\Pr(C_{t+1,i}\ge \mu(1-\sqrt{\alpha})/2e)\le e^{-\Omega(\lambda)}$.
    \end{itemize}
\end{itemize}
\end{lemma}

\begin{proof}
The assumption implies that the two marginals in the 
block $i$ will be set to $C_{t,i}/\mu$ when 
updating the model in iteration $t$.
Statement (A.1) is trivial since 
\begin{align*}
    \mathbb{E}_t[C_{t+1,i}]
    &=\mathbb{E}_{t}[C_{t+1,i-1}]\cdot p_{t+1,2i-1}p_{t+1,2i}\\
&=\lambda (1-1/n)^{2(i-1)}(C_{t,i}/\mu)(C_{t,i}/\mu) \\
&<\lambda(C_{t,i}/\mu)(\mu/\lambda)(1-\varepsilon)\\
&=C_{t,i}(1-\varepsilon).
\end{align*}
Noting also that $C_{t,i}/\mu\le (\mu/\lambda)(1-\varepsilon)<(1-\varepsilon)/2e <(1-\varepsilon)/2$.
The statement (A.2) can be shown as follows.
\begin{align*}
    \mathbb{E}_t[C_{t+1,i}+D_{t+1,i}]
    &=\lambda(1-1/n)^{2(i-1)} ((C_{t,i}/\mu)^2+(1-C_{t,i}/\mu)^2)\\
    &\ge(\lambda/e) (1-2(C_{t,i}/\mu)(1-C_{t,i}/\mu))\\
    &\ge\lambda (1-(1-\varepsilon)(1-(1-\varepsilon)/2))\\
    &=(\lambda/2e)(1+\varepsilon^2)\\
    &>\mu(1+\varepsilon^2).
\end{align*}

For the statement (A.3), we now associate
each of the $\lambda$ individuals in the population 
with an indicator random variable , which is set to 1
if the individual has $i-1$ leading 11s, followed by either a 11
or a 00. There are $\lambda$ such indicators, 
and we are interested in their sum, which is identical to 
the sum of $D_{t+1,i}+C_{t+1,i}$.
By statement (A.2),
the expectation of the sum is at least 
$\mu(1+\varepsilon^2)=\mu/(1-\delta)$ for some 
constants $\varepsilon\in (0,1)$ and 
$\delta\coloneqq 1-1/(1+\varepsilon^2)$. Then, by a Chernoff bound 
(see Lemma~\ref{lemma:chernoff-bound} in the Appendix) 
the probability that the sum is at most 
$(1-\delta)\cdot \mu/(1-\delta)=\mu$ is at most 
$e^{-(\delta^2/2)\cdot \mu/(1-\delta)}=e^{-\Omega(\mu)}$.

For the statement (A.4), we note that $C_{t+1,i}$
is stochastically dominated by another random variable 
$\Tilde{C}$, which is binomially distributed with 
$\lambda$ trials and success probability 
$(\mu/\lambda)^2(1-\varepsilon)^2$. Note that
$\mathbb{E}[\Tilde{C}]=(\mu^2/\lambda)(1-\varepsilon)^2
=\Omega(\mu^2/\lambda)$; thus, 
we can rewrite 
$(\mu^2/\lambda)(1-\varepsilon)
=\mathbb{E}[\Tilde{C}]/(1-\varepsilon)=(1+\varepsilon')\mathbb{E}[\Tilde{C}]$
for some other constant 
$\varepsilon'\coloneqq 1/(1-\varepsilon)-1>0$. By a Chernoff bound, we then obtain
\begin{align*}
    \Pr(C_{t+1,i}\ge (\mu^2/\lambda)(1-\varepsilon)) 
    &\le \Pr(\Tilde{C}\ge (\mu^2/\lambda)(1-\varepsilon))\\
    &=\Pr(\Tilde{C}\ge (1+\varepsilon')\mathbb{E}[\Tilde{C}])\\
    &\le e^{-(\varepsilon')^2\cdot \mathbb{E}[\Tilde{C}]/3}\\
    &=e^{-\Omega(\mu^2/\lambda)},
\end{align*}
which completes proof of statement (A.4).

To prove statement (B.1), by (1) and (2), we get 
\begin{align*}
C_{t,i}/\mu 
&\le (1/2e)(1-\sqrt{2e(1+\varepsilon)(\mu/\lambda)-1})\\
&\le (1/2e)(1-\sqrt{2e(1+\varepsilon)(1/2e)-1}) \\
&= (1-\sqrt{\varepsilon})/(2e)\\
&<(1-\sqrt{\varepsilon})/2.
\end{align*}
Therefore, statement (B.1) can be shown as follows.
    \begin{align*}
    \mathbb{E}_t[C_{t+1,i}]
    &<\lambda(C_{t,i}/\mu)^2 \\
    &=(\lambda/\mu)(C_{t,i}/\mu)C_{t,i}\\
    &\le (2e)((1-\sqrt{\varepsilon})/2e)C_{t,i}\\
    &=C_{t,i}(1-\sqrt{\varepsilon}). 
    \end{align*}
    For statement (B.2), 
    we 
    note  that $C_{t,i}/\mu \le (1-\sqrt{\alpha})/2e<(1-\sqrt{\alpha})/2$ and then obtain
    \begin{align*}
    \mathbb{E}_t[C_{t+1,i}+D_{t+1,i}]
    &\ge (\lambda/e) (1-2(C_{t,i}/\mu)(1-C_{t,i}/\mu))\\
    &>(\lambda/e)(1-(1-\sqrt{\alpha})(1-(1-\sqrt{\alpha})/2)))\\
    &=(\lambda/e) (1-(1-\sqrt{\alpha})+(1/2)(1-\sqrt{\alpha})^2)\\
    &=(\lambda/e)(\sqrt{\alpha} +(1/2)(1-2\sqrt{\alpha}+\alpha))\\
    &=(\lambda/2e)(1+\alpha)\\
    &=(\lambda/2e)2e(1+\varepsilon)(\mu/\lambda) \\
    &= \mu(1+\varepsilon). 
    \end{align*}
    The statement  (B.3)
    follows similarly to the proof of statement (A.3). 
    For the statement (B.4), we employ 
    a similar approach used in (A.4), where we choose 
    $\tilde{C}\sim \bin{\lambda,((1-\sqrt{\alpha})/2e)^2}$ and 
    $\mathbb{E}[\tilde{C}]=\lambda((1-\sqrt{\alpha})/2e)^2
    =\Omega(\lambda)$. We also note that 
    \begin{align*}
    \mu(1-\sqrt{\alpha})/2e
    &=\mathbb{E}[\tilde{C}]/((\lambda/\mu)((1-\sqrt{\alpha})/2e))\\
    &\ge \mathbb{E}[\tilde{C}]/(1-\sqrt{\alpha})\\
    &\ge \mathbb{E}[\tilde{C}]/(1-\varepsilon)\\
    &=(1+\varepsilon')\mathbb{E}[\tilde{C}]
    \end{align*}
    for some other constant 
    $\varepsilon'=1/(1-\varepsilon)-1>0$ and 
    $\alpha\ge \varepsilon$. Since $\tilde{C}$ stochastically dominates 
    $C_{t+1,i}$, we get by a Chernoff bound that
    \begin{align*}
        \Pr(C_{t+1,i}\ge \mu(1-\sqrt{\alpha})/2e)
        &\le \Pr(\tilde{C}\ge \mu (1-\sqrt{\alpha})/2e)\\
         &\le \Pr(\tilde{C}\ge (1+\varepsilon') \mathbb{E}[\tilde{C}])\\
        & \le e^{-(\varepsilon')^2\cdot \mathbb{E}[\tilde{C}]/3}\\
        &\le e^{-\Omega(\lambda)},
    \end{align*}
    which completes the proof.
\end{proof}

We note that combining the two conditions of the statement (A) in
Lemma~\ref{lemma:selective-pressure-less-1-2} yields $C_{t,i}<\mu(1-\varepsilon)/2e$ 
for some small constant $\varepsilon\in (0,1)$, and the same calculation for the 
statement (B) yields 
$C_{t,i}<\mu(1-\sqrt{\varepsilon})/2e$. We observe that 
the two upper bounds are asymptotically
identical, and since the statement (A) considers the 
case of high selective pressure, 
we will in the remainder of the paper
form our arguments based on this result 
only. Previous studies
\cite{Dang:2015,bib:Lehre2019}
show that the \umda only works under a sufficiently high selective pressure (on the
\los function). In other words, if we can show that 
the \umda cannot optimise the \dlblock function
efficiently for some selective pressure $\gamma^*<1/2e$, this result
will highly likely hold for any selective pressure $\gamma^*\ge 1/2e$.

Furthermore, Lemma~\ref{lemma:selective-pressure-less-1-2} also
tells us that in an iteration $t\in \mathbb{N}$ 
if the block $i=Z_t+1$ consists  of 00s and 11s only among the 
$\mu$ fittest individuals, and the number of 11s
is below a threshold $C_{t,i}\le (\mu^2/\lambda)(1-\varepsilon)$ for some
small constant $\varepsilon\in (0,1)$, then in expectation 
the number of 11s sampled 
in the next iteration will shrink, while that of 00s will expand, and the 
block $i$ still consists of 00s and 11s only among the $\mu$ fittest individuals. 
Mathematically speaking, we obtain  
$$
\mathbb{E}_t[C_{t,i}-C_{t+1,i}]>\varepsilon C_{t,i}.
$$
This means that
there is a (multiplicative) drift towards the value of 
zero on the stochastic
process $(C_{t,i})_{t\in \mathbb{N}}$. By the 
multiplicative drift theorem \cite{Doerr:2010:MDA},
the random variable  $C_{t,i}$ will hit the value of zero in
an $\bigO{\log \mu}$ expected time.
Once this  has happened, we will show that
the \umda requires at least $e^{\Omega(\mu)}$ 
iterations in expectation to sample at least 
\begin{equation}
\label{eq:bound-theta}
    \theta
    \coloneqq  (\mu^2/\lambda)(1-\varepsilon)=\gamma^*\mu(1-\varepsilon)
\end{equation}
11s to gain enough momentum to 
escape the `trap' in the block $Z_t+1$ (another way of saying this is to
repair the specified  block). 

Furthermore, we note so far that 
Lemma~\ref{lemma:selective-pressure-less-1-2}  assumes 
that there are only 11s and 00s among the 
$\mu$ fittest individuals in block $i=Z_t+1$, which make 
the two corresponding 
marginals simultaneously set to $C_{t,i}/\mu$.
This is, however, not strictly necessary because
the \umda updates each marginal using the total number of 1-bits
sampled in the bit position, so as long as the 
number of 1s in each bit position is still
below the threshold $\theta$, then all results in
Lemma~\ref{lemma:selective-pressure-less-1-2} still hold.

Recall that we aim at showing an $e^{\Omega(\mu)}$ lower bound on the 
runtime of the \umda on the \dlblock function. This lower bound will be obtained 
if we can show that there exists  a block $i=Z_t+1<m$  
between $Z_0$ and $m=n/2$,
where the two marginals are deceived to reach the lower bound $1/n$, and then
the \umda has to wait a long time (in terms of iterations)
while the block is being repaired. 

\begin{lemma}
\label{lemma:los-expectation-of-Y}
Let $c\log n\le \mu =o(n)$ 
for some sufficiently large
constant $c>0$. 
If there exists a constant $k<m$ such that $Z_t\le k-2$ for any time 
$t\in \mathbb{N}$, then 
for any $j\in [2k-1,n]$ that
\begin{itemize}
    \item[(a)]  $\mathbb{E}\left[p_{t,j}\right] = 1/2$,
  \item[(b)] $\mathbb{E}\left[X_{t,j}\right] 
=  \mu/2$, and
    \item[(c)] $\var[X_{t,j}] \ge 
    (\mu^2/4)(1-o(1))(1-(1-1/\mu)^t)$.
\end{itemize}
\end{lemma}

\begin{proof}
For readability, we omit the index $j$ through out the proof. Recall that $p_{t}=\max\{1/n,\min\{1-1/n,X_{t-1}/\mu\}\}$. 
By the definition of expectation, we get
\begin{align}
\label{eq:expectation-p-t+1}
\begin{split}
\mathbb{E}[p_{t}]
&=(1/n)\cdot \Pr\left(X_{t-1}=0\right)
+\left(1-1/n\right)\cdot\Pr\left(X_{t-1}\mu\right)\\
&\quad +\sum_{k=1}^{\mu-1} (k/\mu)\cdot \Pr\left(X_{t-1}=k\right).
\end{split}
\end{align}
We note further that 
\begin{align*}
    \mathbb{E}[X_{t-1}] 
&= \sum_{k=0}^\mu k \Pr(X_{t-1}=k)= \mu  \Pr(X_{t-1}=\mu)+
\sum_{k=1}^{\mu-1}k \Pr(X_{t-1}=k),
\end{align*}
from which we then obtain
\begin{equation}
\label{eq:combined-k-Prob-X-t-}
    \sum_{k=1}^{\mu-1}k\cdot \Pr(X_{t-1}=k)
    = \mathbb{E}[X_{t-1}] - \mu\cdot \Pr(X_{t-1}=\mu).
\end{equation}
Substituting (\ref{eq:combined-k-Prob-X-t-}) into 
(\ref{eq:expectation-p-t+1}) yields
\begin{align}
\label{eq:expectation-q-t+1}   
\begin{split}
\mathbb{E}[p_{t}]
&=(1/\mu) \mathbb{E}\left[X_{t-1}\right]+(1/n) \left(\Pr\left(X_{t-1}=0\right)-\Pr\left(X_{t-1}=\mu\right)\right).
\end{split}
\end{align}
We note by Lemma~\ref{lemma:distributions-X-Y-f-indpdent} that $X_t$ follows
a binomial distribution with $\mu$ trials and success probability 
$p_{t,j}$, which means that there is no bias towards any border in the stochastic process $(X_{t})_{t\in \mathbb{N}}$. Due to this symmetry, we get
\begin{equation}
\label{eq:symmetry-properties}
    \Pr\left(X_{t-1}=\mu\right) = \Pr\left(X_{t-1}=0\right).
\end{equation}
Furthermore, by the tower rule we also have
\begin{align}
\label{eq:expectation-X-t}
\begin{split}
\mathbb{E}[X_{t-1}]
&=\mathbb{E}[\mathbb{E}[X_{t-1}\mid p_{t-1}]]=\mathbb{E}\left[\mathbb{E}\left[\bin{\mu,p_{t-1}}\mid p_{t-1}\right]\right]
=\mu\cdot  \mathbb{E}[p_{t-1}]
\end{split}
\end{align}
Substituting (\ref{eq:symmetry-properties}) and 
(\ref{eq:expectation-X-t})
into (\ref{eq:expectation-q-t+1}) yields
$
\mathbb{E}[p_{t}] = \mathbb{E}[p_{t-1}]$. 
Then by induction on time, we obtain 
$$
\mathbb{E}[p_{t}] =\mathbb{E}[p_{t-1}]
=\mathbb{E}[p_{t-2}] = \ldots =\mathbb{E}[p_0]=1/2,
$$
which completes the proof of statement (a). 

Statement (b) follows from (\ref{eq:expectation-X-t}) that
$$
\mathbb{E}[X_t] = \mu\cdot \mathbb{E}[p_{t}]=\mu/2.
$$

For statement (c), we note by the tower rule and Lemma~\ref{lemma:second-moment} that
\begin{align}
\label{eq:one}
\begin{split}
    \mathbb{E}[X_{t}^2] &= \mathbb{E}[\mathbb{E}[X_{t}^2\mid p_t]]
    = \mathbb{E}[\mu p_t(p_t(\mu-1)+1)]\\
    &= \mu(\mu-1)\mathbb{E}[p_t^2]+\mu \mathbb{E}[p_t]
    = \mu(\mu-1)\mathbb{E}[p_t^2]+\mu/2.
\end{split}
\end{align}
By the definition of expectation, we also have
\begin{align*}
    \mathbb{E}[p_t^2] &= (1/n)^2\cdot \Pr(X_{t-1}=0) + (1-1/n)^2\cdot \Pr(X_{t-1}=\mu)\\
    & \quad + \sum_{k=1}^{\mu-1}(k/\mu)^2\cdot \Pr(X_{t-1}=k),
\end{align*}
which by noting that 
$$
\mathbb{E}[X_{t-1}^2]= \mu^2 \cdot \Pr(X_{t-1}=\mu) + \sum_{k=1}^{\mu-1} k^2\cdot \Pr(X_{t-1}=k)
$$
satisfies
\begin{align*}
    \mathbb{E}[p_t^2] &= (1/n)^2\cdot \Pr(X_{t-1}=0) + (1-1/n)^2\cdot \Pr(X_{t-1}=\mu)\\
    & \quad + (1/\mu^2)(\mathbb{E}[X_{t-1}^2]-\mu^2 \Pr(X_{t-1}=\mu))\\
    &=(1/\mu^2)\mathbb{E}[X_{t-1}^2] + (1/n)^2\cdot \Pr(X_{t-1}=0)\\
    & \quad + ((1-1/n)^2-1)\cdot \Pr(X_{t-1}=\mu)\\
    &=(1/\mu^2)\cdot \mathbb{E}[X_{t-1}^2] - (2/n)(1-1/n)\cdot \Pr(X_{t-1}=\mu)
\end{align*}
By (\ref{eq:symmetry-properties}),
we can simplify the expression above further as follows.
\begin{align}
\label{eq:two}
\begin{split}
    \mathbb{E}[p_t^2] 
    &= (1/\mu^2)\mathbb{E}[X_{t-1}^2] -(2/n)(1-1/n)\cdot \Pr(X_{t-1}=\mu)\\
    &\ge (1/\mu^2)\mathbb{E}[X_{t-1}^2] -(2/n)(1-1/n)
    \end{split}
\end{align}
since $\Pr(X_{t-1}=\mu)\le 1$. 
Substituting (\ref{eq:two}) into (\ref{eq:one}) yields
\begin{align*}
    \mathbb{E}[X_t^2]
    &\ge (1-1/\mu)\mathbb{E}[X_{t-1}^2] 
    - 2(\mu/n)(\mu-1)(1-1/n) + \mu/2\\
    &= (1-1/\mu)\mathbb{E}[X_{t-1}^2] + (\mu/2)(1-o(1))
\end{align*}
since $\mu=o(n)$.
We now obtain a 
recurrence relation for the expectation of $X_t^2$ w.r.t. time 
$t$ and by  
$\sum_{i=1}^n c^i=(c^{n+1}-1)/(c-1)$ for any $c \neq 1$, we then get
\begin{align*}
\mathbb{E}[X_t^2]
    &\ge  (1-1/\mu)^t\mathbb{E}[X_{0}^2] + (\mu/2)(1-o(1)) \sum_{i=0}^{t-1}(1-1/\mu)^i\\
       &=  (1-1/\mu)^t\mathbb{E}[X_{0}^2] +(\mu/2)(1-o(1))\cdot \frac{(1-1/\mu)^{t}-1}{(1-1/\mu)-1}\\
       &=  (1-1/\mu)^t\mathbb{E}[X_{0}^2] +(\mu^2/2)(1-o(1))(1-(1-1/\mu)^{t})
\end{align*}
which by  
$\mathbb{E}[X_0^2]=\mu(1/2)((1/2)(\mu-1)+1)=\mu(\mu+1)/4$  (see Lemma~\ref{lemma:second-moment} for $X_0\sim \bin{\mu, 1/2}$)
satisfies 
\begin{align*}
          &\ge   (1-1/\mu)^t\mu(\mu+1)/4 +(\mu^2/2)(1- o(1))(1-(1-1/\mu)^{t})\\
          &= (\mu^2/2)(1-o(1))  - (1-1/\mu)^t((\mu^2/2)(1-o(1))-\mu(\mu+1)/4)\\
           &= (\mu^2/2)(1-o(1))  - (1-1/\mu)^t(\mu^2/4)(1-o(1))
\end{align*}
Thus, we obtain
\begin{align*}
    \var[X_{t}] 
    &= \mathbb{E}[X_t^2]-\mathbb{E}[X_t]^2\\
    & \ge   (\mu^2/2)(1-o(1))  - (1-1/\mu)^t(\mu^2/4)(1-o(1)) - (\mu/2)^2\\
    & \ge   (\mu^2/4)(1-o(1))  - (1-1/\mu)^t(\mu^2/4)(1-o(1))\\
     & \ge   (\mu^2/4)(1-o(1))(1  - (1-1/\mu)^t),
\end{align*}
which completes the proof of statement (c).
\end{proof}

One should not confuse the result implied by the statement (a) 
in Lemma~\ref{lemma:los-expectation-of-Y}
with the actual value of the marginals in an arbitrary 
iteration $t\in \mathbb{N}$. For the \umda without borders,
\citet{bib:Friedrich2016} showed that  
even when the expectation stays at
$1/2$, the actual value of the marginal in iteration $t$
can be close to the trivial lower or upper border due to 
its large variance. Recently,
\citet{Zheng:2018:WPB} showed that this actually happens
within an $\Theta(\mu)$ expected number of iterations.

Furthermore, in case of no borders, 
the variance of $X_{t,j}$ for any bit $j\in [2k-1,n]$, where $k$ is defined in 
Lemma~\ref{lemma:los-expectation-of-Y}, can be 
expressed as a function of time as 
$\var^*[X_{t,j}]=(1-(1-1/\mu)^t)(\mu^2/4)$, which can be derived by applying 
the law of total 
variance on the martingale $X_{t,j}\sim\bin{\mu, X_{t-1,j}/\mu}$ (see 
\cite[Lemma~7, Corrolary~9]{bib:Friedrich2016} for a similar derivation for
the \cga without borders). Surprisingly, the statement (c) in 
Lemma~\ref{lemma:los-expectation-of-Y} 
tells us that the variance 
of $X_{t,j}$ for the \umda with borders
grows asymptotically in the same order as 
the variance for the \umda without borders.
The following lemma
shows that after a sufficiently long time the 
probability that one will find the value of 
the random variable 
$X_{t,j}$ smaller than the threshold $\theta$ or greater than 
$\mu-\theta$, where $\theta$ is defined 
in (\ref{eq:bound-theta}), with a constant probability.

\begin{lemma}
  \label{lemma:constant-probability-below-threshold}
  For any $c\in\mathbb{R}$ and even number $\mu\in\mathbb{N}$,
  let $X\in\{0,\ldots,\mu\}$ be a random variable
  with $\var[X]\geq c\mu^2$ and
  \[\prob{X=\mu/2-i}=\prob{X=\mu/2+i}\] for all $i\in[\mu/2]$. Then,
  \begin{align*}
    \prob{X\leq\theta} \geq \frac{2(c-(1/2-\gamma^*)^2)}{1-4(1/2-\gamma^*)^2}.
  \end{align*}
\end{lemma}
\begin{proof}
  Due to symmetry, we have $\expect{X}=\mu/2$ and
  \begin{align*}
    p:=\prob{X\leq\theta} = \prob{X\geq \mu-\theta}.
  \end{align*}
  The lower bound on the variance of $X$ implies
  \begin{align*}
    c\mu^2  \leq \var[X] 
            & = \sum_{i=0}^{\mu} \prob{X=i}(i-\expect{X})^2\\
           & \leq \prob{X\leq \theta} \frac{\mu^2}{4} +
               \prob{\theta<X<\mu-\theta} (\mu/2-\theta)^2 \\
               &\quad +
             \prob{X \geq \mu-\theta}\frac{\mu^2}{4}\\
           & = \frac{p\mu^2}{4} + (1-2p)(\mu/2-\theta)^2 + \frac{p\mu^2}{4}\\
           &= (\mu/2-\theta)^2 + (\mu^2/2-2(\mu/2-\theta)^2)p
  \end{align*}
  Solving the inequality above for $p$ gives
  \begin{align*}
       p & \geq \frac{2(c\mu^2-(\mu/2-\theta)^2)}{\mu^2-4(\mu/2-\theta^2}
        = \frac{2(c-(1/2-\gamma^*)^2)}{1-4(1/2-\gamma^*)^2}.\qedhere
  \end{align*}
\end{proof}

\subsection{Exponential runtime in case of low selective pressure}

In this section, we will show that the \umda 
requires an $e^{\Omega(\mu)}$ expected runtime to optimise
the \dlblock function.
To proceed, we will consider two phases:
\begin{itemize}
    \item[1)] until the algorithm gets stuck at a block $Z_t+1$, and
    \item[2)] while the algorithm is getting stuck and afterwards.
\end{itemize} 
Recall that the algorithm gets stuck at some value $Z_t<m$
when the number of 
1s in each bit position in the block $Z_t+1$ 
among the $\mu$ fittest individuals  is 
at most $\theta=(\mu^2/\lambda)(1-\varepsilon)$ for some
small constant $\varepsilon\in (0,1)$.
We will shows that 
phase 1 will last for $\Omega(\mu)$
iterations. 
The following lemma shows that the 
all-ones bitstring cannot be sampled during 
the first $\Omega(\mu)$ iterations with high probability.

\begin{lemma}
\label{lemma:cannot-sampled-opt-during phase-1}
With probability $1-2^{-\Omega(n)}$, 
the all-ones bitstring cannot be sampled during the 
first $2.92\mu$ iterations of the 
\umda optimising the \dlblock function.
\end{lemma}

\begin{proof}
The proof is inspired by \cite[Lemma~9]{bib:Krejca}.
We will  upper bound the probability of an arbitrary bit position
$j\in \mathbb{N}\cap [2Z_0+3,n]$ 
exceeding $99/100$ 
during the first $\Omega(\mu)$ iterations. 
For readability, we omit the index $j$ and consider the potential 
$\phi_t\coloneqq X_{t}^2$. 
By Lemma~\ref{lemma:distributions-X-Y-f-indpdent}, we can pessimistically assume 
we are not at the borders, which implies  that 
$X_{t+1}\sim \bin{\mu,X_{t}/\mu}$, and by
Lemma~\ref{lemma:second-moment} the expected single-step change is 
\begin{align*}
\mathbb{E}_t[\phi_{t+1}-\phi_t] 
&=\mu \frac{X_t}{\mu} \left(\frac{X_t}{\mu}(\mu-1)+1\right)-X_t^2=X_{t} \left(1-\frac{X_{t}}{\mu}\right)
<\frac{\mu}{4}.
\end{align*}
Let $T\coloneqq  \min\{t\in \mathbb{N}\mid X_{t}\ge 99\mu/100\}$, i.e.,
the first hitting time
of the value of $99\mu/100$ in the stochastic process $(X_{t})_{t\in \mathbb{N}}$.
We have a Markov chain $\phi_T$ with 
process $\phi_t=X_{t}^2$ starting at $(\mu/2)^2$ 
and then progressing by $\phi_{t+1}-\phi_t$ for $T$ iterations.
We then get
\begin{displaymath}
\mathbb{E}[\phi_T] 
=\left(\frac{\mu}{2}\right)^2
+\sum_{t=0}^{T-1}\mathbb{E}[\mathbb{E}_t[\phi_{t+1}-\phi_t]]
<\left(\frac{\mu}{2}\right)^2 + T\cdot \frac{\mu}{4}.
\end{displaymath}
Using Markov's inequality for $k>1$ yields
\begin{displaymath}
\Pr\left(\phi_T \ge k\left(\frac{\mu^2}{4}+T\cdot \frac{\mu}{4}\right)\right)
\le \Pr\left(\phi_T\ge k\cdot \mathbb{E}\left[\phi_T\right]\right)
\le \frac{1}{k}.
\end{displaymath}
We want that $(99\mu/100)^2 \ge k\left(\mu^2/4+T\cdot \mu/4\right)$
since then 
$$
\Pr\left(\phi_T \ge \left(\frac{99\mu}{100}\right)^2\right) 
\le 
\Pr\left(\phi_T\ge k\left(\frac{\mu^2}{4}+T\cdot \frac{\mu}{4}\right)\right)
\le \frac{1}{k}.
$$ 
We get $T\le \mu((4/k)(99/100)^2-1)$, which is positive 
as long as $k\in (1,4\cdot (99/100)^2)$. Thus, we can 
choose a value of $k=1.00001$ and then obtain
$T\le 2.92\mu$.
During the first
$\Omega(\mu)$ iterations, 
the probability of an arbitrary 
marginal $j\in \mathbb{N}\cap [2Z_0+3,n]$ to exceed 
$99/100$ is at most a constant $1/k<1$.
By Lemma~\ref{lemma:fittest-individual}, in expectation there are at most 
$(1/k)\cdot (n-\mathbb{E}[Z_0])=(1/k)\cdot (n-\bigO{\log (\lambda-\mu)}) <n/k$ 
marginals exceeding $99/100$,
and by Chernoff bound, with probability $1-2^{-\Omega(n)}$ there are at most 
$(1+\delta)(n/k)$ such marginals for a constant $\delta>0$
and at least $(1-(1+\delta)/k)n = \Omega(n)$ 
marginals are still below 
$99/100$ during the first $2.92\mu$ iterations. 
Thus, the probability of sampling the all-ones bitstring
is upper bounded by
$(99/100)^{\Omega(n)} = 2^{-\Omega(n)}$.
\end{proof}

We now show that the all-ones bitstring cannot be sampled with 
probability $1-2^{-\Omega(n)}$ while the \umda is getting 
stuck at a block. To do this, 
we need the following two lemmas, where the first
one implies the independent sampling
at the $2(m-(Z_t+1))$ remaining bit positions, and the other states that
the expected values of the marginals of these remaining bits
will highly likely stay around the value of $1/2$.

\begin{lemma}
\label{lemma:Y-independent}
Consider the situation of Lemma~\ref{lemma:los-expectation-of-Y}.
Then, the events of sampling of 11s 
in any block from $k$ to $m$ are pairwise independent. Furthermore, 
the all-ones bitstring cannot be sampled with probability at least 
$1-2^{-\Omega(n)}$. 
\end{lemma}

\begin{proof}
Recall the definition of the constant $k$ in Lemma~\ref{lemma:los-expectation-of-Y}.
We have observed in Lemma~\ref{lemma:distributions-X-Y-f-indpdent} that 
the number of 11s sampled among the $\mu$ fittest 
individuals in an arbitrary 
block $j\ge k$ in iteration $t\in \mathbb{N}$ 
is binomially distributed with
$\mu$ trials and success probability $p_{t,2j-1}p_{t,2j}$. 
This result also implies that sampling a 11 at a block $j_1$ is 
independent of sampling a 11 at a block $j_2$, for any $j_1,j_2\in [k,m]$ and 
$j_1\neq j_2$, which proves the first claim.

For the second claim, 
we prove by considering the number of 11s 
sampled in an offspring
between blocks $k$ and $m$. 
By Lemma~\ref{lemma:los-expectation-of-Y},
the probability of sampling a 11 in a block is 
$1/4$, so the expected number of 11s sampled between blocks
$k$ and $m$ (i.e., there are $m-k+1=\Theta(n)$ 
blocks in total) is given by
\begin{align*}
\mathbb{E}\left[Y_t\right]
=(m-k+1)\cdot (1/2)^2 
=\Theta(n).
\end{align*}
Then, by the
Chernoff-Hoeffding bound, 
the probability of sampling these blocks all 
as 11s is at most $e^{-\Omega(n)}$.
\end{proof}

We are now ready to prove an exponential runtime of the \umda on
function \dlblock when the selective pressure is $\gamma^*=\Omega(1)$.

\begin{theorem}
\label{thm:exponential-runtime}
The expected runtime of 
the \umda with the parent population size 
$c\log n \le \mu =o(n)$ 
for some sufficiently large constant $c>0$, and
parent and offspring population sizes satisfying
$\frac{\mu}{\lambda}>\frac{14}{1000}$  is $e^{\Omega(\mu)}$ 
on the \dlblock function.
\end{theorem}

\begin{proof}
The all-ones bitstring cannot be sampled  during the first 
$\Omega(\mu)$ iterations with probability $1-2^{-\Omega(n)}$ 
(by Lemma~\ref{lemma:cannot-sampled-opt-during phase-1}).
After roughly $2.92\mu$ iterations, we obtain from
Lemma~\ref{lemma:los-expectation-of-Y} that 
$\var[X_t]\ge (1-o(1))(\mu^2/4)(1-1/e^{2.92})$.  
If we choose the 
selective pressure $\gamma^*>\frac{14}{1000}$,  then
by Lemma~\ref{lemma:constant-probability-below-threshold} for
$c=(1/4)(1-1/e^{2.92})(1-o(1))$, we obtain a constant lower
bound on the probability that $X_t$ drops below the threshhold value 
$\theta$, defined in (\ref{eq:bound-theta}). 
This also means that 
we can find with probability $\Omega(1)$ 
that in any bit position in
the block $Z_t+2$ there are fewer than $\theta$ 
1s sampled, where $\theta$ is defined 
in (\ref{eq:bound-theta}). 
Assume in iteration $t'=t+1$ that $Z_{t'}=Z_t+1$,
then with probability $\Omega(1)$ the \umda will get
stuck at block $Z_{t'}+1=Z_t+2$ in iteration $t'$. 
Assume now that we are in iteration $t'$, 
there is a multiplicative drift towards 
the value of zero in the stochastic process $(C_{t'+\Delta t,i})_{\Delta t\in \mathbb{N}}$
in block $i\coloneqq Z_{t'}+1$. By multiplicative drift theorem, 
the number of 11s sampled 
there will reduce to zero within an $\bigO{\log \mu}$ 
expected number of iterations.
In particular, we note by the statements (A.3) and  (A.4) of 
Lemma~\ref{lemma:selective-pressure-less-1-2} that after the 
number of 11s has 
dropped below the threshold $\theta$, 
the event of sampling at least 
$\theta$ 11s in the next iteration in block $i$ or sampling
at least one 01 or 10
among the $\mu$ fittest individuals 
is at most 
$$
e^{-\Omega(\mu)}+ e^{-\Omega(\mu^2/\lambda)}\le 
e^{-\Omega(\gamma^*\mu)} =e^{-\Omega(\mu)}
$$
since we consider only $\gamma^*>\frac{14}{1000}$.
Thus, the \umda requires at least 
$1/e^{-\Omega(\mu)}=e^{\Omega(\mu)}$
iterations in expectation until this block has been repaired.
But then, the algorithm
will likely get stuck in some of the following blocks.
By Lemma~\ref{lemma:fittest-individual}, 
the expected number of times the algorithm gets trapped is
$$
\Omega(1)\cdot (n/2-\mathbb{E}[Z_0]) 
\ge \Omega(1)\cdot (n-\bigO{\log (\lambda-\mu)})
=\Omega(n).
$$
Therefore, the expected number of iterations 
of the \umda on the \dlblock function is 
at least $\Omega(n)\cdot e^{\Omega(\mu)}$.
We also note that so far in the proof we have always assumed 
the co-occurrence of the following two events:
\begin{itemize}
    \item[(A)] The all-ones bitstring will not
    be sampled during the first $\Omega(\mu)$
iterations with probability $1-2^{-\Omega(n)}$ 
(by Lemma~\ref{lemma:cannot-sampled-opt-during phase-1}).
    \item[(B)] The all-ones bitstring cannot be sampled 
    while the \umda is getting stuck
at a block with probability $1-2^{-\Omega(n)}$ (by 
Lemma~\ref{lemma:Y-independent}).
\end{itemize}
By union bound,
the success probability is still $1-2^{-\Omega(n)}$. By the law of total
expectation \cite{bib:Motwani1995} and noting further that the \umda performs
$\lambda$ function evaluations in every iteration, 
the overall expected runtime is lower-bounded by
\begin{displaymath}
    (1-2^{-\Omega(n)})\cdot 
    \lambda \cdot (\Omega(\mu)+ \Omega(n)\cdot 
e^{\Omega(\mu)}) 
=e^{\Omega(\mu)}.\qedhere
\end{displaymath}
\end{proof}

We observe that if the parent population size is 
$\mu=\Theta(\log n)$, Theorem~\ref{thm:exponential-runtime}
yields a lower bound of $n^{\Omega(1)}$ for any selective 
pressure $\gamma^*>\frac{14}{1000}$. The intuition is that 
when the population size $\mu$ is small.
the threshold 
$\theta$ is not too large, meaning that
the \umda only needs to sample a few 1-bits in order to escape the 
`current trap' in an arbitrary iteration.
Larger population sizes, such as $\mu =\Omega(n^\varepsilon)$ 
for some constant $\varepsilon>0$, will result in 
an exponential lower bound of $2^{\Omega(n^\varepsilon)}$ 
on the expected runtime of the 
\umda on the \dlblock function. 

\subsection{Extremely high selective pressure may help}
We now apply Theorem~\ref{thm:levelbasedtheorem} to
derive an $\bigO{n^3}$ expected runtime
for the \umda on the \dlblock function
under extremely high selective pressures of
$\gamma^*=\bigO{1/\mu}$ (or $\lambda=\Omega(\mu^2)$).

\begin{theorem}\label{thm:umda-on-dlb-high-sel-press}
    The \umda with the parent 
    population size $\mu \geq c \log{n}$ for a 
    sufficiently large constant $c>0$, and the offspring population size
    $\lambda \ge (1+\delta)e\mu^2$ 
    for any constant $\delta>0$, 
    has expected optimisation time 
    $\mathcal{O}(n\lambda\log{\lambda}+n^3)$ on the 
    \dlblock function.
\end{theorem}
\begin{proof}
Recall that 
levels are defined as in (\ref{level-definition}).
    There are $m \coloneqq   (n/2)+1$ 
    levels from $A_0$ to $A_{m}$. 
     
     For condition (G2), for any level 
    $j \in \{0\}\cup [m-2]$ satisfying $|P_{t} \cap 
    A_{\geq j}| \geq \gamma_0\lambda = \mu$ and 
    $|P_{t} \cap A_{\geq j+1}| \geq \lambda 
    \gamma$ for some $\gamma \in (0,\gamma_0]$, 
    we seek a lower bound $(1+\delta)\gamma$ 
    for $\prob{y \in A_{\geq j+1}}$ where 
    $y$ is sampled from the model $p_{t+1}$. 
    The given conditions 
    on $j$ imply that the $\mu$ fittest individuals of
    $P_t$ have at least $j$ leading 11s 
    and among them at least $\lceil \gamma\lambda \rceil$ 
    have at least $j+1$ leading 11s. 
    Hence, $p_{t+1,i}=1-1/n$ for $i \in [2j]$, 
    and for $i\in \{2j+1, 2j+2\}$ that 
    $p_{t+1,i} \geq \max(\min(1-1/n,
    \gamma\lambda/\mu),1/n)\geq \min(1-1/n,\gamma/\gamma_0)$, so
    \begin{align*}
    \prob{y \in A_{\geq j+1}}
        &\geq \prod_{i=1}^{j+1} p_{t+1,i}
        \geq \left(1 - \frac{1}{n}\right)^{2j} 
        \left(\frac{\gamma}{\gamma_0}\right)^2\\
        &\ge  \frac{1}{e}\left(\frac{\gamma}{\gamma_0}\right)^2
        = \frac{\lambda\gamma}{e\mu^2} \geq (1 +\delta)\gamma
    \end{align*}
    due to $\gamma_0 \geq \gamma\geq 1/\lambda$,
    and $\lambda\ge (1+\delta)e\mu^2$
    for any constant $\delta>0$. Therefore, condition (G2)  is now satisfied.
    
    For condition (G1), 
    for any level $j \in \mathbb{N}\cap [0,m-1]$ satisfying $|P_{t} \cap A_{\geq j}
    | \geq \gamma_0\lambda = \mu$ we seek a lower bound $z_j$ on $\prob{y \in 
        A_{\geq j+1}}$. 
        Again the condition on level $j$ implies that
    $p_{t+1,i}=1-1/n$ for $i \in [2j]$. 
    Due to the imposed lower margin, 
    we can assume pessimistically that 
    $p_{t+1,2j+1}=p_{t+1,2j+2} \ge 1/n$. 
    Hence, 
    \begin{align*}
    \prob{y \in A_{\geq j+1}}
        &\geq \left(1 - \frac{1}{n}\right)^{2j} \left(\frac{1}{n}\right)^2 
        \ge \frac{1}{en^2} =: z_j.
    \end{align*}
    So, (G1) is satisfied for $z_j\coloneqq  1/(en^2)$.
    
    Considering (G3), because $\delta$ is a constant, and
    both $1/z_*$ and $m$ are $\mathcal{O}(n)$, 
    there must exist a constant $c>0$ such that 
    $\mu \geq c \log n \geq (4/\delta^2)\ln(128 m / (z_* \delta^2))$. Note that 
    $\lambda = \mu/\gamma_0$, so (G3) is satisfied.
    
    All conditions of Theorem~\ref{thm:levelbasedtheorem} are 
    satisfied, so the expected optimisation 
    time of the \umda on the \dlblock function is
    \begin{displaymath}
        \mathcal{O}\left(\sum_{j=1}^{n/2} 
        \left( \lambda \ln\lambda + n^2\right)\right)
        = \mathcal{O}\left(n\lambda\log\lambda + n^3\right).\quad\qedhere
    \end{displaymath}
    \end{proof}

One might say that the failure of the \umda to optimise 
the \dlblock problem is not necessarily coming from the pairwise deception, 
but from the low selective pressure in the same way that it struggles on the 
\los function \cite{bib:Lehre2019}. However, as a final remark, we think that this claim 
is inaccurate. In fact, for the range of selective pressures
  considered, it has been shown that the \umda optimises other
  non-deceptive problems easily in 
  polynomial expected runtime, including
  \los \cite{Dang:2015,bib:Lehre2019} and \om \cite{Dang:2015}.

  Non-elitist algorithms using ($\mu,\lambda$)-selection are typically
  efficient when $\mu/\lambda<1/e$, i.e., below the
  error threshold (see \cite{Lehre2010-PPSN}). In contrast, to show
  that the \umda optimises \dlblock efficiently, we need to decrease
  this ratio significantly so that we get
  $\mu/\lambda=\bigO{1/\mu}$, i.e., extremely high selective pressure.
We do not normally see such high selective pressures in practical applications of the \umda. Furthermore, under an 
 $\bigO{1/\mu}$ selective pressure the \umda selects few fittest 
 individuals to update the model. In this extreme situation, we believe that 
 the \umda degenerates into the 
 $(1,\lambda)~\ea$ (see Figure~1 for some empirical result of the \umda with
 extreme selective pressure and simple \eas), where only the fittest
 individual(s) are selected to the next generation.
 
 We have already shown that simple \eas can optimise the \dlblock
 function easily under normal selective pressure. For this reason, we
 think that comparing the \umda and simple \eas in this (extreme) regime of the
 selective pressure is not very interesting.

\section{Experiments}
\label{sec:experiments}

In the previous sections, we showed that the \umda fails to optimise
\dlblock since it does not remember what it has learned so far during
the optimisation process. This is due to the lack of ability to
capture interactions between bits.

To complement the theoretical investigation, we studied empirically the behaviour
of the UMDA in the case of normal selective pressure, and in the case
of extremely high selective pressure. As a base-line, we also
considered the standard $(\mu,\lambda)$ EA. Figure \ref{fig:numblocks} shows
the results of 100 repetitions of the \umda and the
$(\mu,\lambda)$ EA on the DLB problem of size $n=200$. The boxplots
show the interquartile range of the number of correct blocks obtained
by the fittest individual in the population, as a function of the
number of fitness evaluations $t$. The maximal number of correct
blocks is $n/2=100$. Each run was terminated after $10^6$ function
evaluations. The UMDA in the extreme selective pressure setting uses
the parameters $\mu=10$ and $\lambda=1000$. The boxplots show that the
algorithm in this setting converges to a non-optimal equilibirum
position (around 20-30 correct blocks), as predicted by the
theoretical analysis. The UMDA in the normal selective pressure
setting uses the parameters $\mu=200$ and $\lambda=1000$. The boxplots
show that the algorithm in this setting improves steadily, again as
predicted by the theoretical analysis. The $(\mu,\lambda)~\ea$ uses
the parameters $\mu=200, \lambda=1000$, and bitwise mutation rate
$1/n$. The boxplots show that the algorithm has a steady progress, as
predicted by the theoretical analysis.

\begin{figure*}
  \centering
  \includegraphics[width=\linewidth]{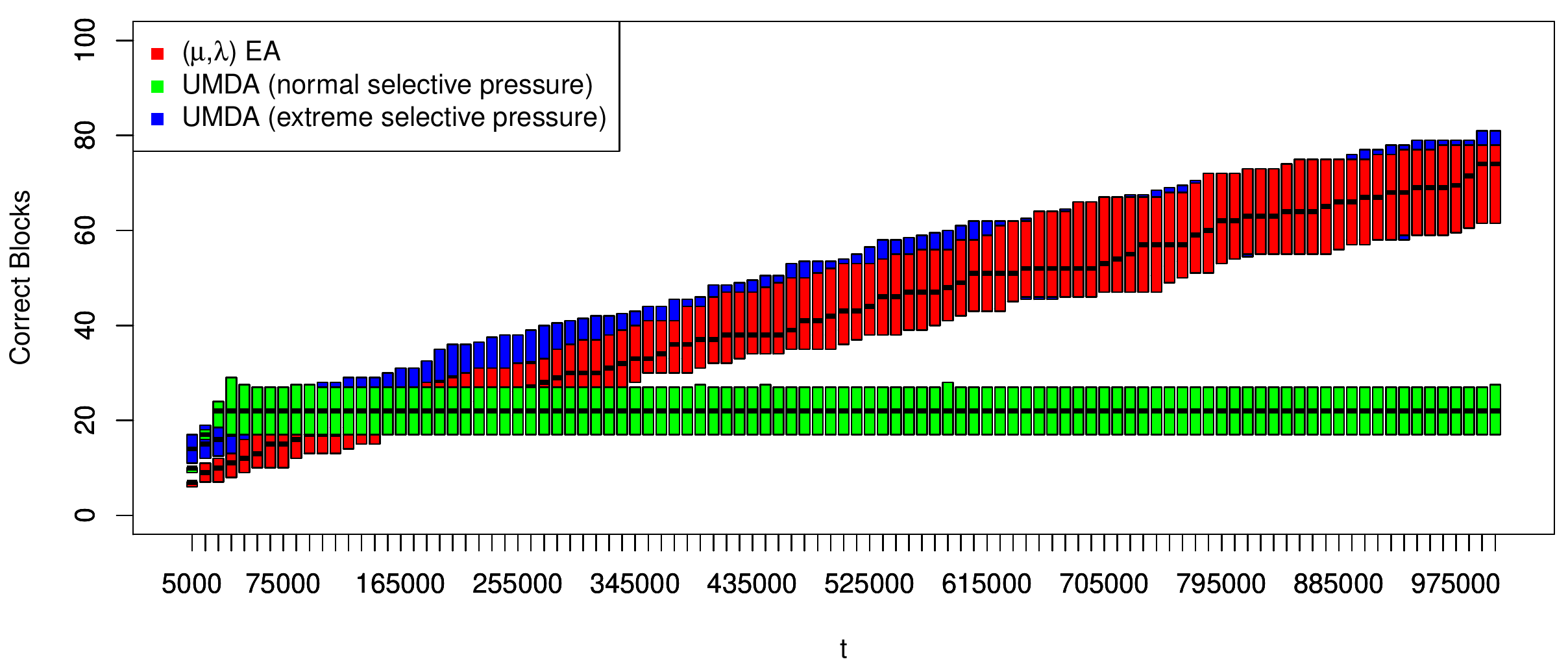}
  \caption{
  Number of correct blocks in the fittest individuals of the $(\mu,\lambda)$
    EA and UMDA on the DLB problem
    with problem size $n=200$ as a function of number of function
    evaluations $t$. UMDA with extreme selective pressure uses
    parent population size $\mu=10$, while the other experiments use
    $\mu=200,$ $\lambda=1000$, and mutation rate $1/n$ for the
    $(\mu,\lambda)$ EA.}
  \label{fig:numblocks}
\end{figure*}

We now look at another class of \edas, which
attempt to learn variable dependencies.
Here, we consider the \mimic algorithm (see Algorithm \ref{mimic-algor}).
The reason behind the inclusion of the \mimic is that
this algorithm builds a chain-structured model in each generation
using entropy and conditional entropy between decision variables. 
In order to optimise \dlblock, algorithms need to correct all blocks from left to right and maintain these blocks over generations. 
Our conjecture is that the \mimic can deal with small traps in \dlblock easily.
This is because, if the $\mu$-th individual in the sorted population has $i$ 
leading 11s, then bits at positions 
$1,2,\ldots,2i$ will have the minimum entropy and conditional entropy, 
which might form the prefix 
of the chain-structured model. 
The algorithm might then be able to 
choose the `right' bits (i.e., bits $2i+1$ and $2i+2$) 
to add to the prefix  and easily increase the 
value of $i$ in the next generation.

We consider three different settings for the population: 
$\lambda = \sqrt{n}$ (i.e. small), $\lambda=\sqrt{n}\log n$ (medium) 
and $\lambda=n$ (large) for $n\in \{10,15,20,\ldots,100\}$.  
As mentioned before, the \mimic
includes a vast number of probability calculations 
(to obtain entropy and conditional entropy), but this does not matter as we are
only interested in the number of fitness evaluations the algorithm performs.
For each value of $n$, the algorithms are run 
100 times and the average runtime is computed. The results are
shown in Figure~\ref{fig:mimic}.

\begin{figure*}[htb]
    \centering
    \subfloat[$\lambda=\sqrt{n}$]{%
        \includegraphics[width=0.315\textwidth]{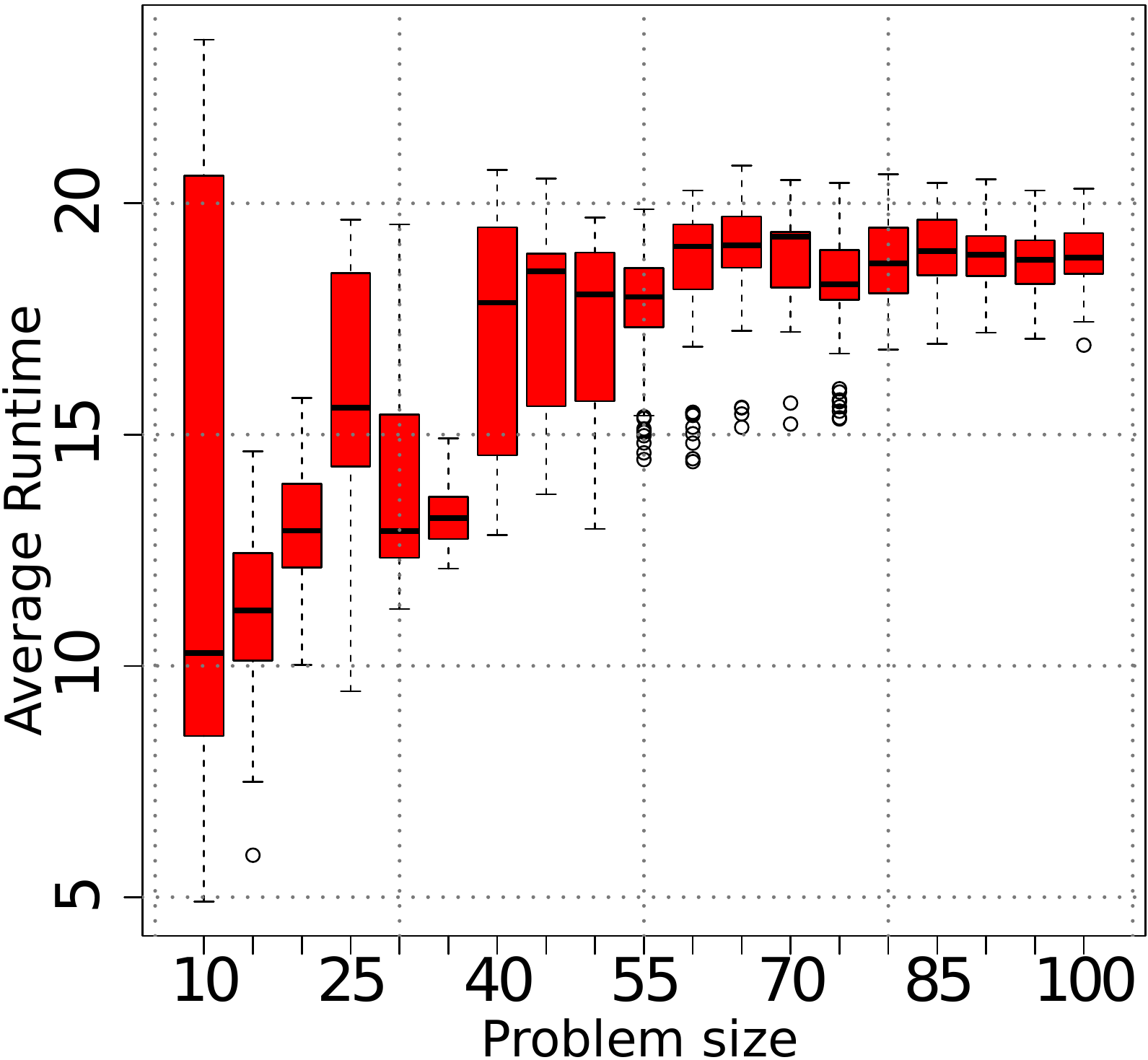}%
        \label{binval-small-mu}%
        }%
    \hfill%
    \subfloat[$\lambda=\sqrt{n}\log n$]{%
        \includegraphics[width=0.315\textwidth]{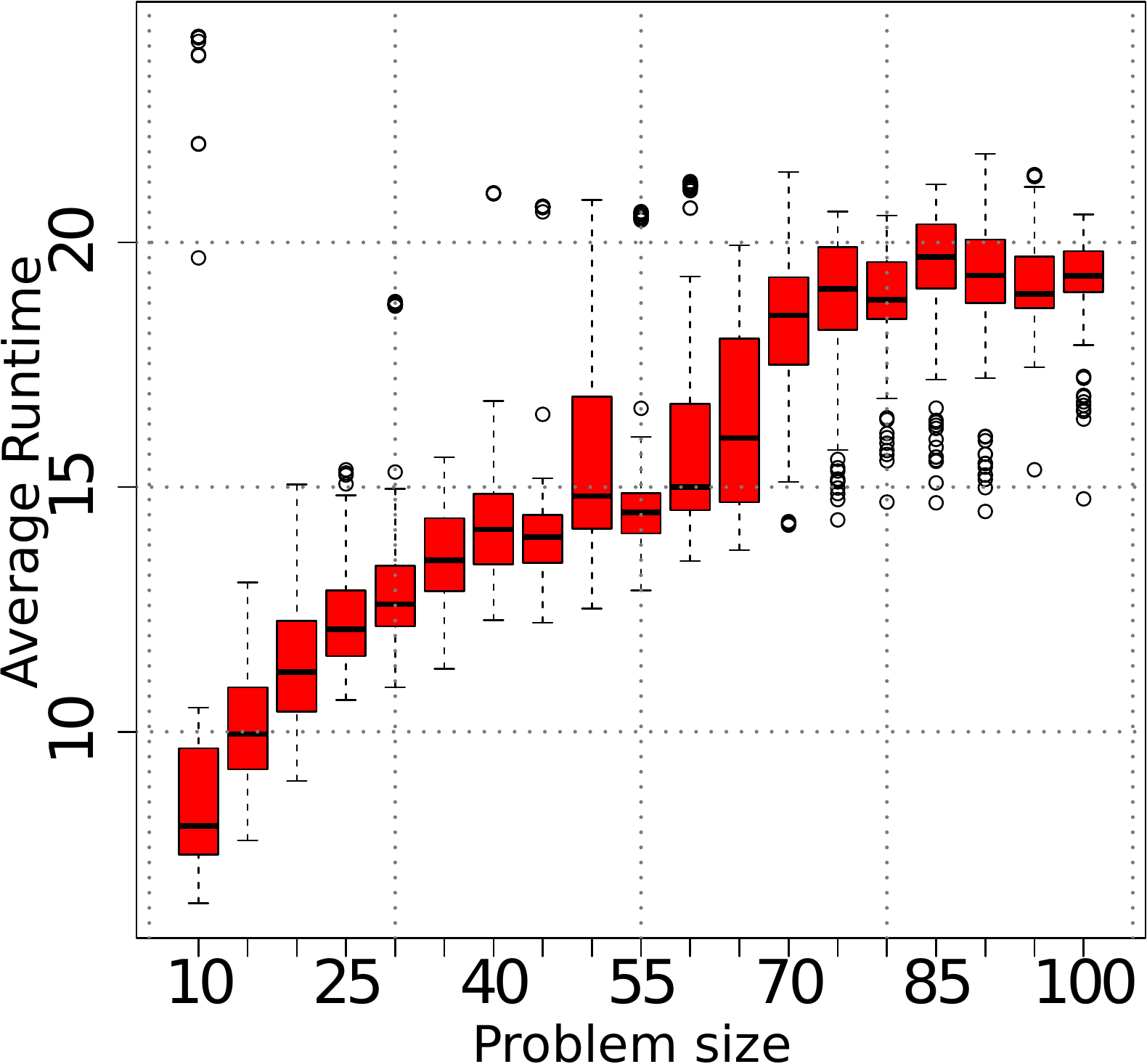}%
        \label{binval-small-mu}%
        }%
    \hfill%
    \subfloat[$\lambda=n$]{%
        \includegraphics[width=0.315\textwidth]{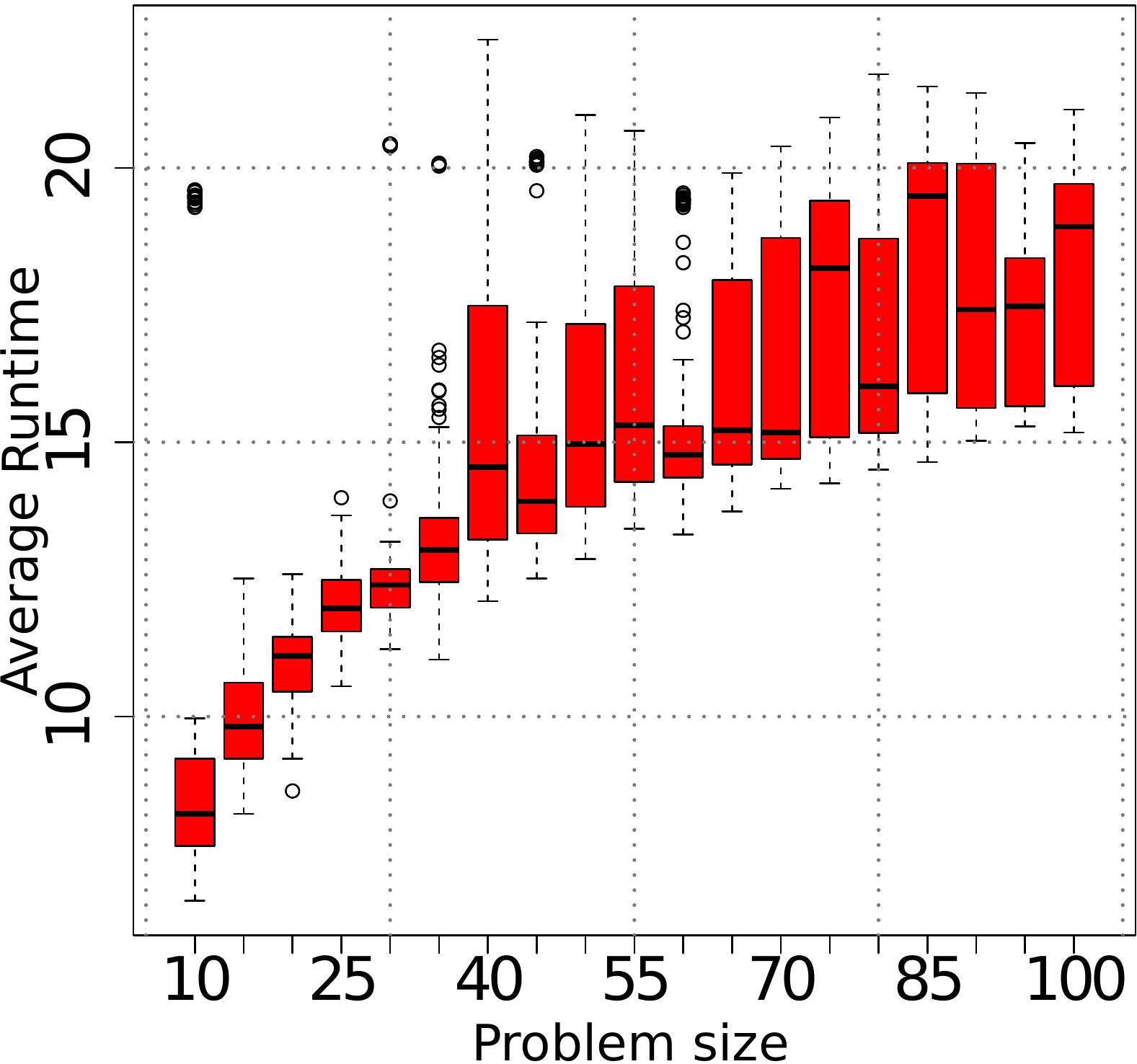}%
        \label{binval-large-mu}%
        }%
    \caption{Average runtime of the \mimic on the \dlblock function. 
    The y-values are the base-2 logarithm of the number of fitness evaluations. 
    There are 100 runtime measures for each value of the problem instance size $n$.}
    \label{fig:mimic}
\end{figure*}

In general, the results show that the \mimic can find the optimum of the
\dlblock problem 
under all choices of population size. When $n$ is small 
(especially around 10), there is a 
very large variance in the number of fitness evaluations. 
This is because the parent population is so small, the model is easily influenced by the few top individuals, which may result in a random walk. When the population size becomes larger, the diversity in the population
can be maintained for a longer period of time, 
and the empirical marginals cannot deviate too far from
the true marginals; thus, the algorithm looks more stable.  
Furthermore, the size of the population seems to 
have a real impact on the runtime of
the algorithm. For small and medium size 
(i.e. only differ by a logarithmic factor), the runtime looks similar;
however, for large $\lambda=n$ the \mimic on average 
requires a smaller number of fitness 
evaluations but exhibits a larger variance.
Although we cannot provide a rigorous analysis of the \mimic on the \dlblock problem,
the shapes in all figures suggest that the number of fitness
evaluations might be polynomial in problem size $n$. 

\section{Conclusions}
\label{sec:conclusion}

In this paper, we have introduced the Deceptive Leading Blocks, in which 
bits are highly correlated (similar to \los) 
and contains many small traps (like a trap function). 
Since this function is new, we first show that simple  \eas
can optimise the function in polynomial time. More specifically, we
consider three typical \eas, 
that are $(1+\lambda)~\ea$, $(\mu+1)~\ea$ and the non-elitist $(\mu,\lambda)~\ea$, and the upper bounds derived 
verified our claims. Next, we aim at showing 
that due to correlation and deception,  the
\dlblock problem may be hard for the \umda, which assumes 
independence between variables. We are able to show a
lower bound of $e^{\Omega(\mu)}$ 
on the expected runtime
of the \umda on \dlblock for any selective pressures $\frac{\mu}{\lambda}>\frac{14}{1000}$
and parent population size
$\mu=\Omega(\log n)$ and $\mu=o(n)$. On the other hand, if the selective
pressure is extremely high (i.e., $\bigO{1/\mu}$), the \umda optimises
the \dlblock function in an $\bigO{n\lambda\log \lambda + n^3}$ expected runtime.

We believe that the difficulty that the \umda faces when optimising \dlblock stems from the first-order statistics underlying the algorithm, that causes the algorithm to forget everything it has learned so far. Motivated by this observation, we look at the class of bivariate \edas, and a typical one is the \mimic. 
Due to the complexity of how the entropy and pair-wise conditional entropy measured to build a chain-structured model, we were unable to analyse it using rigorous arguments. Thus, we present some experimental results of the \mimic on
\dlblock to draw future attention from the community. The experimental findings
suggest that the \mimic might be able to optimise \dlblock 
in polynomial runtime and exhibit the ability to fill in 
the gaps left by the \umda.

We leave it as an open problem for future work to analyse the
behaviour of the algorithm in the case
$\frac{14}{1000}>\frac{\mu}{\lambda}=\omega(1/\mu)$.  Future work should also
consider theoretical aspects of the \mimic on simple toy functions
like \om and \los. Rigorous arguments should be formed to describe how
the chain-structured model is built and how the next population is
sampled using pairwise conditional probabilities following the order
of bits in the model.


\bibliography{references}
\bibliographystyle{abbrvnat}
\end{document}

%% file: preamble.tex
\usepackage{amssymb}
\usepackage{microtype}
\usepackage[linesnumbered,ruled,vlined]{algorithm2e}
\usepackage{amsfonts}
\usepackage{mathtools}
\usepackage{amsmath}
\usepackage{amsthm}
\usepackage{dsfont}
\usepackage{float}
\usepackage{booktabs}

\usepackage[square,numbers]{natbib}

\newtheorem{theorem}{\textbf{Theorem}}

\newtheorem{lemma}{\textbf{Lemma}}

\usepackage{graphics}
\usepackage{bm}  
\usepackage{mathtools}
\usepackage{wrapfig}
\usepackage{subfig}

\DeclarePairedDelimiter\floor{\lfloor}{\rfloor}
\usepackage{xspace}
\usepackage{flushend}
\usepackage{epstopdf}
\usepackage{enumitem}

\DeclareMathOperator{\Ber}{Bernoulli}       

\usepackage{xcolor}

\newcommand{\var}{\text{Var}} 

\usepackage{threeparttable}

\newcommand{\oneoneEA}{$(1+1)~\text{EA}$\xspace}
\newcommand{\onelambdaEA}{$(1+\lambda)~\ea$\xspace}
\newcommand{\muoneEA}{$(\mu+1)~\ea$\xspace}

\newcommand{\pbil}{\text{\sc PBIL}\xspace} 
\newcommand{\mimic}{\text{\sc MIMIC}\xspace} 
\newcommand{\eda}{\text{\sc EDA}\xspace} 
\newcommand{\edas}{\text{\sc EDAs}\xspace} 
\newcommand{\ea}{\text{\sc EA}\xspace} 
\newcommand{\eas}{\text{\sc EAs}\xspace} 
\newcommand{\cga}{\text{\sc cGA}\xspace} 
\newcommand{\umda}{\text{\sc UMDA}\xspace} 
\newcommand{\om}{\text{\sc OneMax}\xspace} 
\newcommand{\bval}{\text{\sc BinVal}\xspace} 
\newcommand{\los}{\text{\sc LeadingOnes}\xspace} 

\newcommand{\dlblock}{\text{\sc DLB}\xspace} 

\newcommand{\bin}[1]{\text{Bin}\left(#1\right)} 

\newcommand{\Natural}{\mathbb{N}} 
\newcommand{\sspace}{\mathcal{X}}

\newcommand{\prob}[1]{\Pr\left(#1\right)}
\newcommand{\expect}[1]{\mathbb{E}\left[#1\right]}

\newcommand{\bigO}[1]{\mathcal{O}\left(#1\right)} 
